\documentclass[final]{cvpr}

\usepackage{times}
\usepackage{epsfig}
\usepackage{graphicx}
\usepackage{amsmath}
\usepackage{amssymb}
\usepackage{url}            
\usepackage{booktabs}       
\usepackage{amsfonts}       
\usepackage{nicefrac}       
\usepackage{microtype}      
\usepackage{lipsum}
\usepackage{color,enumitem}
\usepackage{amsthm}
\usepackage{bm}
\usepackage[short]{optidef}

\usepackage{algpseudocode}
\usepackage{algorithm}

\usepackage{subfigure}
\newtheorem{definition}{Definition}
\usepackage{bbm}
\usepackage[short]{optidef}
\usepackage{tabularx,ragged2e}
\usepackage{multirow}

\newboolean{showcomments}
\setboolean{showcomments}{true}

\expandafter\def\expandafter\normalsize\expandafter{%
    \normalsize
    \setlength\abovedisplayskip{2pt}
    \setlength\belowdisplayskip{2pt}
    \setlength\abovedisplayshortskip{0pt}
    \setlength\belowdisplayshortskip{0pt}
}

\newcommand{\Junshan}[1]{  \ifthenelse{\boolean{showcomments}}
	{ \textcolor{red}{(Junshan says:  #1)}} {}  }

\newtheorem{lemma}{Lemma}

\newtheorem{assumption}{Assumption}

\def \n2{{N_0 \over 2}}

\def \tphi{\tilde{\phi}}
\def \ttheta{\tilde{\theta}}
\def \tw{\tilde{w}}

\def\prox{\mbox{\textrm{prox}}}

\def \h5{\hspace{0.5in}}

\usepackage{thmtools,thm-restate}

\usepackage[pagebackref=true,breaklinks=true,colorlinks,bookmarks=false]{hyperref}

\usepackage[capitalise]{cleveref}

\def\cvprPaperID{11008} 
\def\confYear{CVPR 2021}

\begin{document}

\title{MetaGater: Fast Learning of Conditional Channel Gated Networks via  Federated Meta-Learning}

\author{
Sen Lin\textsuperscript{\rm 1},
Li Yang\textsuperscript{\rm 1},
Zhezhi He\textsuperscript{\rm 2},
Deliang Fan\textsuperscript{\rm 1},
Junshan Zhang \textsuperscript{\rm 1}

\\\textsuperscript{\rm 1}School of ECEE, Arizona State University\\

\textsuperscript{\rm 2}Department of Computer Science and Engineering, Shanghai Jiao Tong University\\
\{slin70, lyang166, dfan, junshan.zhang\}@asu.edu, zhezhi.he@sjtu.edu.cn
}

\maketitle

\begin{abstract}
 While deep learning has achieved phenomenal successes in many AI applications, its enormous model size and intensive computation requirements pose a formidable challenge to the deployment in   resource-limited nodes.  There has recently been an increasing interest in computationally-efficient learning methods, e.g., quantization, pruning and channel gating. However, most existing techniques cannot  adapt to  different tasks quickly. In this work, we advocate a holistic approach to jointly train  the backbone network  and the channel gating which enables dynamical selection of a subset of filters for more efficient local computation given the data input. Particularly, we develop a federated meta-learning approach to jointly learn good meta-initializations for both backbone networks and gating modules, by making use of the model similarity across learning tasks on different nodes. In this way, the learnt meta-gating module effectively captures the important filters of a good meta-backbone network, based on which  a task-specific conditional channel gated network can be quickly adapted, i.e., through one-step gradient descent, from the meta-initializations in a two-stage procedure using new samples of that task. The convergence of the proposed federated meta-learning algorithm is established under mild conditions. Experimental results corroborate the effectiveness of our method in comparison to related work.
\end{abstract}

\section{Introduction}

The last decade has witnessed an explosive boost in deep learning, especially Deep Neural Networks (DNN), leading to phenomenal successes in many artificial intelligence applications, e.g., speech recognition~\cite{lecun2015deep}, image classification~\cite{he2016deep,rawat2017deep}, object detection~\cite{szegedy2013deep,lin2017focal} and etc. Nevertheless, DNNs nowadays have very complex structures, e.g., a larger model depth and width for a more expressive representation power, which is not amendable to the deployment in  resource-constrained settings (e.g., edge servers or robots \cite{lin2020edge}). This challenge has spurred significant effort on computationally-efficient learning methods recently, including weight quantization~\cite{han2015deep,hubara2017quantized,he2019simultaneously}, weight pruning~\cite{han2015learning,wen2016learning,yang20dac} and channel gating~\cite{wang2018skipnet,chen2019you,abati2020conditional}. Notably, both weight pruning and channel gating aim to effectively select only a  portion of model parameters, i.e., a sub-network, for local computation (inference) with minimal performance loss, through a sampling mask on either the weights directly, or the channels.


\begin{figure}
\centering
\includegraphics[scale=0.143]{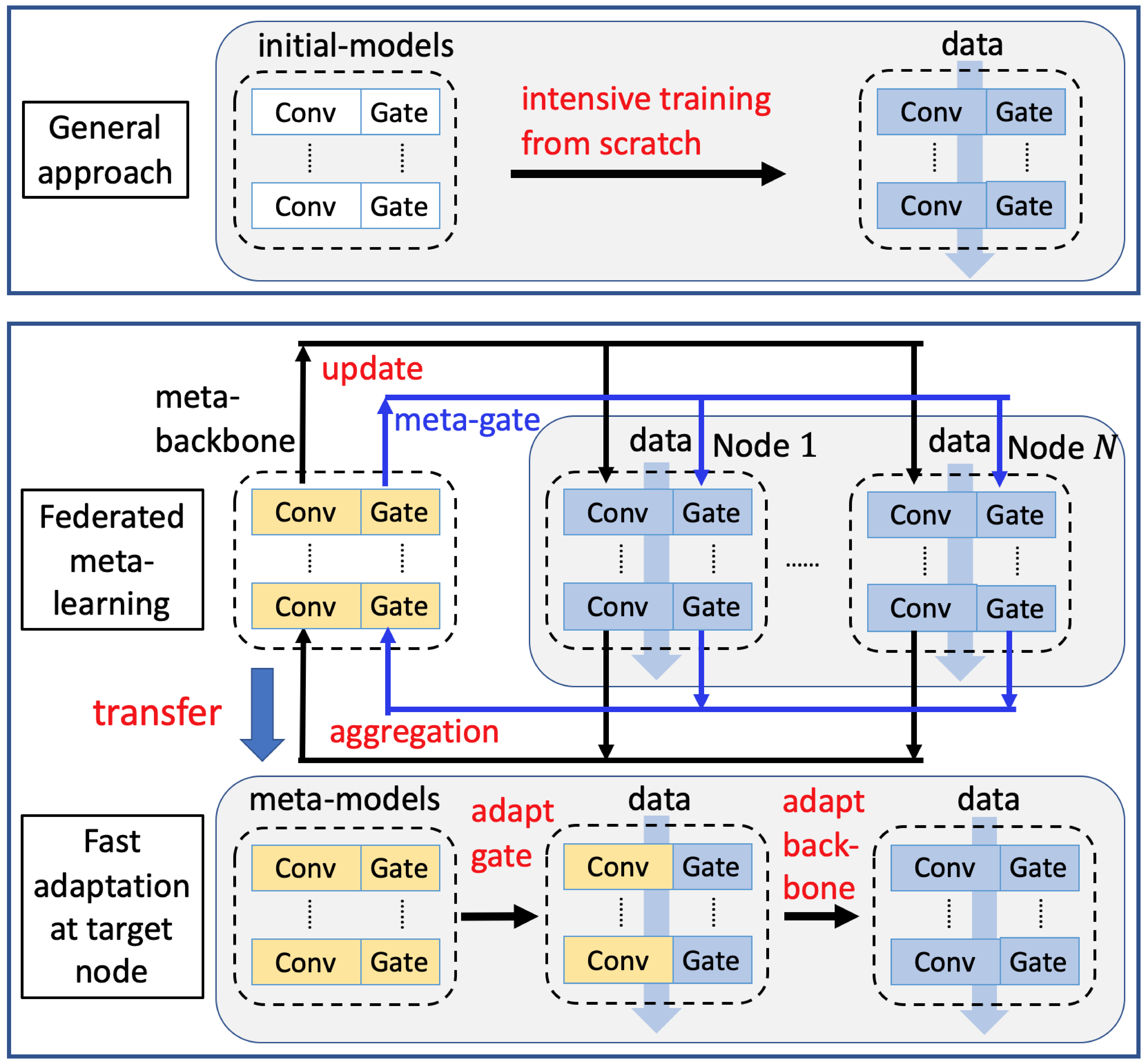}
\caption{Illustration of  the proposed framework MetaGater. 
Top: general learning of conditional channel gated networks. Bottom: fast learning of conditional channel gated networks via joint federated meta-learning of backbone and gating module.}
\label{Fig:outline}
\vspace{-0.5cm}
\end{figure}

However, most aforementioned studies are afflicted with some of the following limitations: 
1) \emph{Extensive training cost or poor performance.} Generally, the learning of an efficient subnet for a resource-limited node requires substantial pre-training on a large target dataset, which often takes place in a powerful cloud datacenter~\cite{yang20dac}. 
Such a learning strategy, however, may not be practical due to the concerns on cost and privacy, as
a significant amount of data need to be transmitted from the node to the server.
It is worth noting that some recent works~\cite{lee2018snip,ramanujan2020s} propose to quickly prune a randomly initialized DNN, and then fine-tune the subnet for eliminating the need of pre-training, which however suffers from poor performance and may not actually speed up the inference, as a  drawback for unstructured pruning \cite{wen2016learning}. 
(2) \emph{Limited adaptability to different tasks.} The fact that different nodes usually have different local data distributions \cite{mcmahan2017communication,kairouz2019advances,sattler2019robust}, implies that a global background model is not sufficient to guarantee universally satisfactory learning performance across different tasks. Accordingly, the subnet selection should also vary for different tasks. Yet, most works ~\cite{mallya2018piggyback,abati2020conditional} consider a common background model and require non-trivial re-training of masks with massive training data when applied to different tasks.
\emph{In a nutshell, for effective learning at resource-limited nodes (devices),
it is desirable for the learning of subnets for each new task to be able to quickly adapt while incurring minimum training cost, akin to cognitive learning by human beings.}

To tackle these challenges, we propose MetaGater, a fast learning framework for conditional channel gated networks by leveraging the knowledge from many nodes, where the backbone network model at each node   is associated with a task-specific \emph{channel gating module}. 
The channel gating module can generate a data-dependent mask, i.e., a binary vector, for each layer in the backbone network, which dynamically selects a subset of \emph{filters} to participate into the computation conditioned on the data input, thereby improving the computation efficiency. 
Since the learning tasks across different nodes often share some similarity~\cite{argyriou2008convex,smith2017federated,aipe2018similarhits}, \emph{we advocate a  federated meta-learning approach to jointly learn good meta-initialization for both the backbone networks and the channel gating modules}, via the collaboration of many nodes in a distributed manner. The learnt meta-backbone network and the meta-gating module are then transferred from the cloud to a target node for fast adaptation (as shown in Fig. \ref{Fig:outline}). Since the meta-gating module effectively captures the important filters of a good meta-backbone network and hence  sparsity structure across tasks, it can  achieve the agile adaptability at different new tasks
by  quickly learning a task-specific conditional channel gated network using corresponding new data samples.
    
    

The main contributions of this work can be summarized as follows.

   (1) To achieve fast and adaptive learning of subnets on resource-limited nodes, we propose MetaGater, a fast learning framework of conditional channel gated networks via federated meta-learning, where good meta-initializations for both backbone networks and gating modules are jointly learnt by leveraging knowledge from related tasks. A task-specific conditional channel gated network for a new task can then be learned quickly  from the meta-initialization in a two-stage procedure, using data samples of the new task.
    
    (2) To efficiently solve the federated meta-learning problem with non-smooth objective functions, we develop a novel approach based on accelerated proximal gradient descent with inexact solutions to the local problems. 
    Particularly, we use accelerated gradient descent for the meta-backbone network and accelerated proximal gradient descent for the meta-gating module, in the same spirit with the Nesterov's method~\cite{nesterov2013introductory}. By characterizing and controlling the estimation error introduced by the inexact solutions, we establish the convergence of the proposed federated meta-learning algorithm for non-convex functions under mild conditions, and show that an $\epsilon$-first order stationary point can be obtained in at most $O(\epsilon^{-1})$ communication rounds. 
    
    (3) We conduct  experiments  to evaluate the effectiveness of MetaGater. Specifically, the experiments on various datasets showcase that the proposed federated meta-learning approach clearly outperforms existing baselines in terms of  accuracy and efficiency. Since this study focuses on the fast learning performance based on distributed learning, most existing methods based on centralized pre-training on a large target dataset cannot directly serve as the baseline. For a fair comparison, we develop a new baseline MetaSNIP by integrating federated meta-learning with one state-of-the-art fast pruning approach SNIP \cite{lee2018snip}, where we apply SNIP to the meta-backbone network obtained by using the  federated meta-learning approach. Our experimental results indicate that MetaGater is able to quickly obtain a task-specific subnet with higher accuracy, and achieves a larger diversity in the task-model sparsity after fast adaptation, compared with MetaSNIP. This implies that MetaGater can successfully find the joint model of the meta-backbone network and meta-gating module that is sensitive to changes in the tasks, such that quick adaptation in the model parameters will lead to a good task-specific channel gated network for efficient inference.

\section{Related Work}

\paragraph{Federated meta-learning.} Meta-learning is a promising solution for fast learning, where one gradient-based meta-learning algorithm called MAML~\cite{finn2017model} has become a representative method. The idea of MAML is to learn a model initialization based on many related tasks, such that even one-step gradient descent from this initialization can achieve good performance for a new task using a few data samples from that task. A lot of work have been proposed to understand~\cite{antoniou2018train,fallah2020convergence,ji2020multi} and improve upon MAML~\cite{li2017meta,nichol2018first,zhou2019efficient}. 

Recently, the marriage of federated learning and meta-learning has garnered a lot of attention, giving rise to a new research direction, namely federated meta-learning. In particular, the empirical successes of such an integration have been corroborated in~\cite{chen2018federated,jiang2019improving}. From a theoretic point of view, the work~\cite{lin2020collaborative} establishes the convergence of federated meta-learning for strongly convex functions and investigates the impact of task similarity. Besides, \cite{fallah2020personalized} studies the case of non-convex functions with stochastic gradient descent. All the works above have studied federated meta-learning based on MAML-type methods. Recently, a different federated meta-learning approach is proposed in~\cite{dinh2020personalized}, based on a proximal meta-learning method with moreau envelopes. However, to our best knowledge, we are the first to \emph{study the agile adaptability and  computational efficiency when federated meta-learning is leveraged to jointly learn the backbone network and the gating module, and further analyze the convergence performance for non-smooth functions in this setting}. More importantly, the federated meta-learning approach proposed in this work clearly outperforms the previous studies as shown in the experiments.

\paragraph{Channel gating and weight pruning.} 
The idea of utilizing data-dependent channel gating module~\cite{bengio2013deep,sigaud2015gated,hua2019channel,gao2018dynamic,hua2019boosting} to improve the computational efficiency has recently emerged in the literature. Specifically, \cite{chen2019you} proposes Gaternet to train a separate gating network to select filters for each layer in the backbone network. To increase the amount of conditional features actually learned, a batch-shaping technique is introduced in~\cite{bejnordi2019batch} to learn conditional channel gated networks. \cite{abati2020conditional} further applies the channel gating module to address the catastrophic forgetting in task-aware continual learning, by predicting the current task in a set of pre-defined tasks. 
The works in network pruning can be traced back to early 1990s~\cite{reed1993pruning}, where sparsity enforcing penalty terms (e.g., $\mathcal{L}_0$ and $\mathcal{L}_1$ norm)~\cite{weigend1991generalization, ishikawa1996structural} and saliency criterions like weight sensitivity~\cite{karnin1990simple} are widely used. Recently, using magnitude of weights~\cite{han2015deep} as the criterion has achieved significant successes and become a standard method for network pruning, which however needs expensive prune-retrain cycles. SNIP~\cite{lee2018snip} proposes a fast pruning method based on connection sensitivity without pre-training. \cite{ramanujan2020s} directly finds optimal subnets in randomly weighted neural networks. Nevertheless, it is difficult for these unstructured pruning methods to reduce the inference time on hardware due to the highly irregular sparsity patterns. Therefore, many approaches in structured pruning~\cite{wen2016learning,li2016pruning,liu2017learning} have been proposed to \emph{prune weights grouped in regular shapes, such as channels or kennels}.

Note that an independent and concurrent work~\cite{song2020rapid} also proposes to utilize meta-learning for rapid structural pruning of neural networks. We highlight the main differences below: 1) \cite{song2020rapid} relies on a centralized meta-learning method where the nodes are required to submit data to a central platform, whereas we consider a more realistic distributed setup and propose a new federated meta-learning approach to fit the specific efficiency problem in our work. 2) \cite{song2020rapid} takes a stochastic approach and learns a task-specific Bernoulli distribution for mask generation, which however could possibly generate masks that lead to significant performance degradation. In stark contrast, we develop a deterministic approach by learning a task-specific channel gating module, and also provide theoretic foundations by carrying out a thorough convergence analysis of the proposed federated meta-learning algorithm.

\section{Methodology}

In this section, we first present the problem formulation for learning task-specific conditional channel gated networks, and then introduce the proposed federated meta-learning for conditional channel gated networks through the collaboration among a set of nodes, followed by the fast adaptation procedure at a target node.

\subsection{Problem formulation}

As alluded to earlier, we integrate a task-specific channel gating module with the backbone network at each node, so as to improve the computational efficiency with a data-dependent channel sampling mask.

\textbf{Backbone network.}
For a node $i$, let $\ttheta^i$ denote the model parameters for 
the backbone neural network with $J$ convolutional layers, which serves as the main model that extracts features from the data input and makes predictions. 

 \begin{figure}
\centering
\includegraphics[scale=0.165]{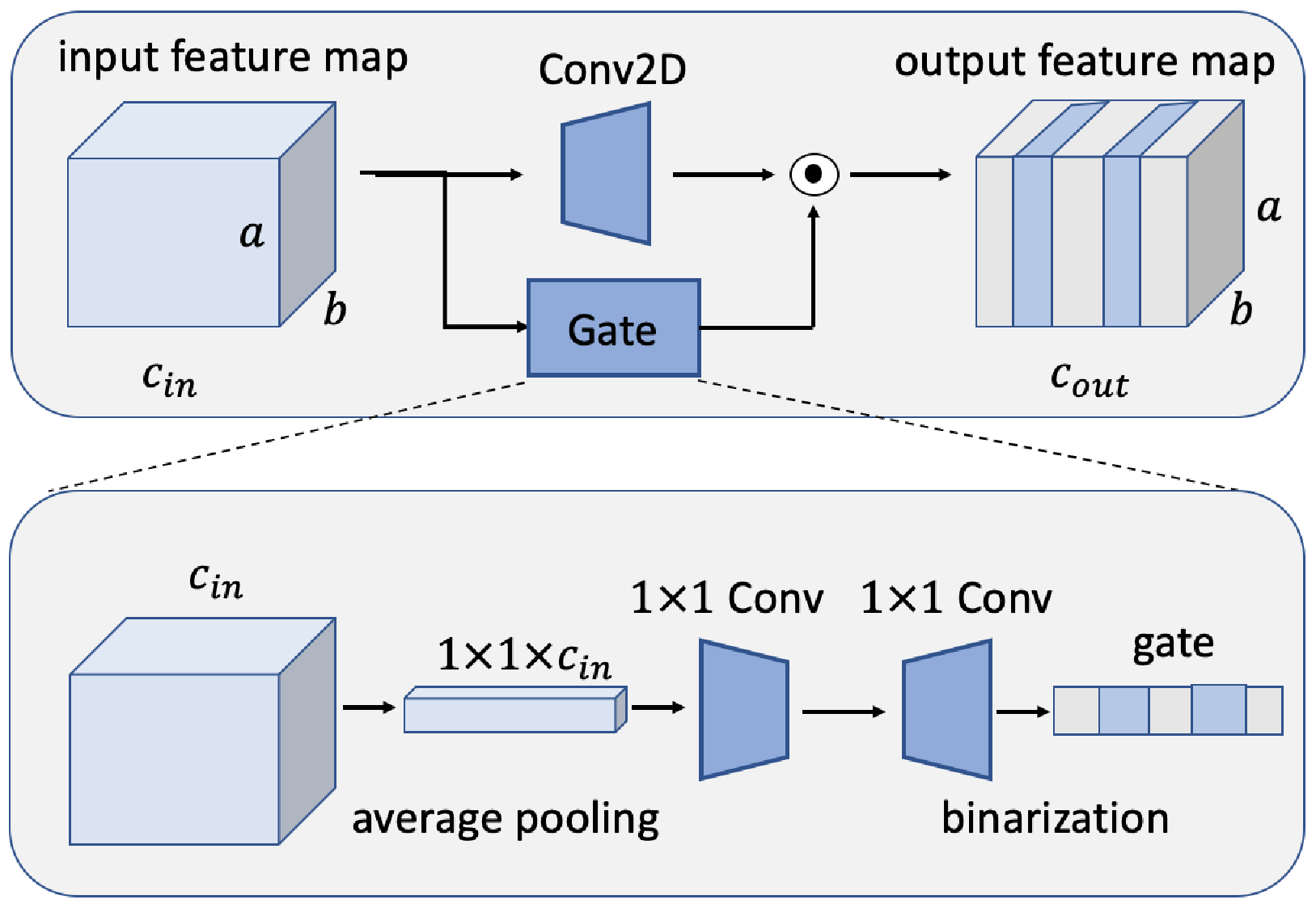}
\caption{The channel gating module for a convolution layer.}
\label{Fig:fml_att}
\vspace{-0.5cm}
\end{figure}

\textbf{Channel gating module.} 
Let $\tphi^i$ denote the model parameters for the channel gating module $Q_i=[Q_i^1, ..., Q_i^J]$ at node $i$. As depicted in~\cref{Fig:fml_att}, let 
$o^j\in\mathbb{R}^{c_\textrm{in}^j,a,b}$ and $o^{j+1}\in\mathbb{R}^{c_\textrm{out}^j,a,b}$ be the input and output feature maps of the $j$-th convolutional layer in the backbone network, respectively. 
Conditioned on the input feature map $o^j$, the layer-wise channel gating module $Q_i^j$ generates a channel mask vector with binary elements, to determine which channels should be activated for the given input. As a result, a sparse feature map $\hat{o}^{j+1}$, instead of $o^{j+1}$, is forwarded to the next layer, only with the channels activated by the gating module $Q_i^j$, i.e.,
{\small
\begin{equation}
    \hat{o}^{j+1}=Q_i^j(o^j)\odot o^{j+1}
\end{equation}
}%
where $Q_i^j(o^j)=[q_1^j,...,q_{c_{out}^j}^j]$, $q_i^j\in\{0,1\}$ and $\odot$ represents the channel-wise multiplication. Each gating module consists of Multi-Layer Perception (MLP) with a single hidden layer featuring 16 units, followed by a ReLU activation function. To generate the binary mask, we utilize the binarization function~\cite{courbariaux2016binarized}, and estimate the gradient via straight-through estimator (STE)~\cite{bengio2013estimating} for the forward and backward paths, respectively. More details about the channel gating module are described in the appendix. 


 
\textbf{Learning of task-specific channel gated networks.}
For a target node $0$, let $L_0(\tphi^0,\ttheta^0)$ denote the empirical loss over the local dataset $D_0=\{(\mathbf{x}^0_k,\mathbf{y}^0_k)\}_{k=1}^K$:
{\small
\begin{equation}\label{loss}
    L_0(\tphi^0,\ttheta^0)\triangleq \frac{1}{|D_0|}\sum_{(\mathbf{x}^0_k,\mathbf{y}^0_k)\in D_0} l(\tphi^0, \ttheta^0;(\mathbf{x}^0_k,\mathbf{y}^0_k))
\end{equation}
}%
for some standard loss $l$, e.g., cross-entropy loss. Then, 
the joint learning of the backbone network and the channel gating module can be formulated as the following regularized optimization problem:
{\small
\begin{equation}\label{localobject}
    \min_{\tphi^0,\ttheta^0} ~ L_0(\tphi^0,\ttheta^0)+\frac{\lambda}{2}\|\tphi^0-\phi\|^2_2+\frac{\lambda}{2}\|\ttheta^0-\theta\|^2_2
\end{equation}
}%
where $\lambda$ is some constant penalty parameter, $\phi$ and $\theta$ are some prior model parameters for the gating module and the backbone network, respectively. Clearly, directly solving \eqref{localobject}, i.e., searching for the optimal task-specific conditional channel gated network, is computationally challenging in general, and possibly suffers from poor performance if only a small local dataset is used for training.
On the other hand, the quality of the regularizer plays an important role in controlling the performance of learnt conditional channel gated network, in the sense that the closer the prior parameters are to the task-specific optimal parameters, the better the learning performance is. This regularized learning problem is also intimate with biased regularized hypothesis transfer learning~\cite{zhou2019efficient}, which has thoroughly demonstrated its efficiency in many applications~\cite{fei2006one,wang2014flexible,kuzborskij2017fast}.
\emph{Therefore, instead of directly solving \eqref{localobject} as in prior works~\cite{chen2019you,abati2020conditional}, we take a different approach by learning a more informative prior regularizer, such that a quick adaptation at the target node through gradient descent can lead to a good approximation of the optimal subnet.}

\subsection{Joint learning of meta-backbone network and meta-gating module via federated meta-learning} 

Since the learning tasks across different nodes usually share some similarity, to obtain a good prior regularizer for~\eqref{localobject}, we develop a new federated meta-learning approach, not only to learn a meta-backbone network but also to learn a meta-gating module, by leveraging the knowledge from related tasks on a set of nodes.

\begin{figure}
\centering
\includegraphics[scale=0.18]{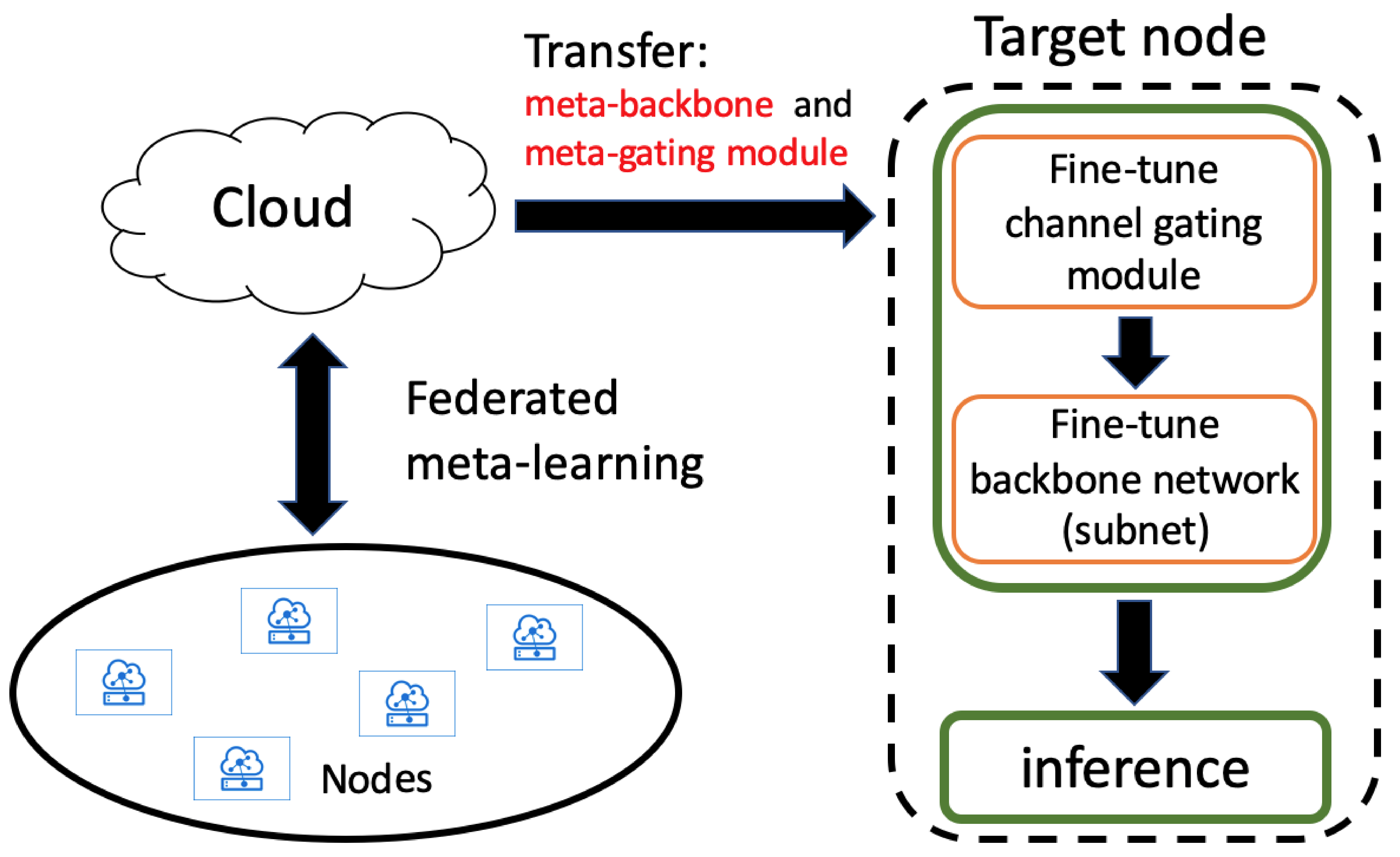}
\caption{A fast learning framework of conditional channel gated networks via federated meta-learning.}
\label{Fig:fml_frame}
\vspace{-0.5cm}
\end{figure}

As illustrated in~\cref{Fig:fml_frame}, a set $\mathcal{S}$ of nodes, each $i\in\mathcal{S}$ with a task and a local dataset $D_i$, participate into federated meta-learning to jointly learn meta-models for both backbone networks and channel gating modules, which are then transferred via the cloud to a target node for fast adaptation. Let $|\mathcal{S}|=N$.
Intuitively, a good meta-model $(\phi,\theta)$ should have the following properties:
\begin{itemize}
    \item For any $i\in\mathcal{S}$,
    the meta-model should be `close'  to its task-specific optimal backbone network and gating module, such that the loss is minimized when solving a similar local problem as~\eqref{localobject}. In this way, the learnt meta-model implicitly captures the way to quickly learn the optimal task-specific conditional channel gated network with local data across all nodes in $\mathcal{S}$. 
    
    \item In general, a quick adaptation through gradient descent is not sufficient to prompt the sparsity of the task-specific gating module. Instead, it is more effective to start with a meta-gating module with structured sparsity (which also serves as the initialization of fast adaptation), for better computational efficiency.
\end{itemize}
Therefore, the objective of federated meta-learning can be mathematically formulated as follows:
{\small
\begin{align}\label{metaobject}
    \min_{\phi,\theta}~F(\phi,\theta)&\hspace{-0.05cm}=\hspace{-0.05cm}H(\phi)+\frac{1}{N}\sum_{i=1}^N \min_{\tphi^i,\ttheta^i}~G_i(\tphi^i,\ttheta^i)\\
    s.t.~G_i(\tphi^i,\ttheta^i;\phi,\theta)&\hspace{-0.05cm}=\hspace{-0.05cm}L_i(\tphi^i,\ttheta^i)+\frac{\lambda}{2}\|\tphi^i-\phi\|^2_2+\frac{\lambda}{2}\|\ttheta^i-\theta\|^2_2,\nonumber
\end{align}
}%
where $L_i(\tphi^i,\ttheta^i)$ is the empirical loss defined in a similar way with \eqref{loss} for node $i$, and
$H(\cdot)$ is some sparsity prompting function of the meta-gating module, such as $\mathcal{L}_1$-norm and Group Lasso~\cite{yuan2006model}. It is worth noting that compared with the popular gradient-based meta-learning method MAML~\cite{finn2017model}, such a regularization-based meta-learning formulation can fully leverage the higher-order information~\cite{zhou2019efficient} of the objective function~\eqref{metaobject}, leading to more informative meta-backbone networks and meta-gating modules.

 \emph{The next key question is how to efficiently solve problem~\eqref{metaobject} in a distributed manner}. To answer this question, there are some problems to be solved: 1) Computationally-efficient methods usually result in  performance degradation, compared with that of the entire backbone network. To guarantee the performance of conditional channel-gated networks after fast adaptation, a better meta-backbone network is needed, given a fixed communication budget (between the cloud and nodes) which is often a significant bottleneck in wireless networks. 
2) Generally $H(\cdot)$ is a non-smooth function such that the classical gradient descent may not work well.

To address the above problems, we develop a new federated meta-learning approach based on accelerated proximal gradient descent \cite{ghadimi2016accelerated}, as summarized in~\cref{alg1}. In what follows, we highlight several key aspects in our algorithm design.
\begin{itemize}
    \item Generally, it is computationally expensive to find a local minimizer to problem \eqref{localmin} (in \cref{alg1}) at each node. Instead, we run the vanilla gradient descent for several steps to approximately solve \eqref{localmin}, for the case when $G_i$ is smooth for a smooth local loss $L_i$. In this way, \cref{alg1} would obtain a meta-model $(\phi,\theta)$ such that the conditional channel gated network, obtained after fast adaptation via several gradient descent steps at each node, can achieve good learning performance.
    \item Unlike the meta-model update in MAML which includes the computation of Hessian, the global update of meta-models in~\cref{alg1} is as easy to implement as first-order meta-learning algorithms, e.g., Reptile~\cite{nichol2018first}. Particularly, let $(\tphi^{i*}_t,\ttheta^{i*}_t)=\arg\min_{(\tphi^i,\ttheta^i)} G_i(\tphi^i_{t},\ttheta^i_{t};\phi_t^{pr},\theta_t^{pr})$. For the $t$-th iterate $(\phi_t^{pr},\theta_t^{pr})$ of the meta-model in \cref{alg1}, it can be shown~\cite{zhou2019efficient,t2020personalized} that $(\tphi^{i*}_t,\ttheta^{i*}_t)$ satisfies:
    {\small
    \begin{align*}
        \lambda(\phi_t^{pr}-\tphi^{i*}_{t})=&\nabla_{\phi_t^{pr}} G_i(\tphi^{i*}_{t},\ttheta^{i*}_{t};\phi_t^{pr},\theta_t^{pr}),\\
        \lambda(\theta^{pr}_t-\ttheta^{i*}_{t})=&\nabla_{\theta_t^{pr}} G_i(\tphi^{i*}_{t},\ttheta^{i*}_{t};\phi_t^{pr},\theta_t^{pr}),
    \end{align*}
    }%
    which indicate that the global updates of meta-model (step 8 and 9 in~\cref{alg1}) follow an approximate gradient direction with respect to (w.r.t.) the meta-learning objective~\eqref{metaobject}.
    \item Since the meta-learning objective $F$ is a non-smooth function w.r.t. $\phi$, we apply proximal gradient descent for the global update of the meta-gating module. More specifically, the proximal operator \cite{parikh2014proximal} of function $H$ is defined by
    \vspace{-0.1cm}
    {\small
    \begin{equation*}
    \vspace{-0.1cm}
        \prox_{\eta H}(v)=\arg\min_x \left(H(x)+\frac{1}{2\eta}\|x-v\|^2_2\right)
    \end{equation*}
    }%
    for $\eta>0$. When $H=0$, it is clear that $\prox_{\eta H}(v)=v$. Hence, if $v$ is a standard gradient step as in step $9$ of Algorithm \ref{alg1}, $\prox_{\eta H}(v)$ can be interpreted as trading off minimizing $H$ and being close to $v$.
    \item Since the global updates of meta-models indeed follow an approximate gradient direction w.r.t \eqref{metaobject} (gradient descent for meta-backbone network and proximal gradient descent for meta-gating module), we resort to a general acceleration technique \cite{ghadimi2016accelerated} for the global updates to improve the performance of federated meta-learning. When $\beta_t=\eta_t\alpha_t$, the global updates of meta-models fall into the variants of the well-known Nesterov's method~\cite{nesterov2013introductory}.
\end{itemize}

\begin{algorithm}[t]
  \small
	\caption{Joint federated meta-learning}
	\label{alg1}
 	\begin{algorithmic}[1]
		  \State Set initial models $\phi_0^{ag}=\phi_0$ and $\theta_0^{ag}=\theta_0$;
		  \For{$t= 1, 2, ..., T$}
		  \State Update $\phi^{pr}_t=\alpha_t \phi_{t-1}+(1-\alpha_t)\phi_{t-1}^{ag}$, $\theta^{pr}_t=\alpha_t \theta_{t-1}+(1-\alpha_t)\theta_{t-1}^{ag}$;
		    \State  Send $\phi^{pr}_t$ and $\theta^{pr}_t$ to the nodes;
		    \For{each node $i\in\mathcal{S}$}
		       \State Update the backbone network and the gating module using gradient descent to achieve an approximate minimizer $(\tphi^i_{t},\ttheta^i_{t})$ to the following problem		      \begin{align}\label{localmin}
	          \min_{\tphi^i,\ttheta^i}~ G_i(\tphi^i,\ttheta^i;\phi_t^{pr},\theta_t^{pr}),
		        \end{align}
		        such that $\|\nabla G_i(\tphi^i_{t},\ttheta^i_{t};\phi_t^{pr},\theta_t^{pr})\|^2\leq \xi_t$;
		   \State Send $(\tphi^i_{t},\ttheta^i_{t})$ to the cloud;
		   \EndFor
		    \State  Globally aggregate $\nabla_{\phi,t}=\frac{\lambda}{N}\sum_{i=1}^N (\phi^{pr}_t-\tphi^i_{t})$, and $\nabla_{\theta,t}=\frac{\lambda}{N}\sum_{i=1}^N (\theta^{pr}_t-\ttheta^i_{t})$;
		   \State Update global models
		   \vspace{-0.2cm}
		    \begin{align*}
		 \phi_{t}=\prox_{\eta_t H}\left(\phi_{t-1}-\eta_t \nabla_{\phi,t}\right),& ~ \theta_t=\theta_{t-1}-\eta_t \nabla_{\theta,t};\\
		 \phi^{ag}_{t}=\prox_{\beta_t H}\left(\phi^{pr}_t-\beta_t \nabla_{\phi,t}\right),& ~ \theta^{ag}_t=\theta^{pr}_{t}-\beta_t \nabla_{\theta,t}.
		 \vspace{-0.2cm}
		    \end{align*}
		\EndFor
		\State Set $\phi\leftarrow \phi^{pr}_T$, $\theta\leftarrow \theta^{pr}_T$.
	\end{algorithmic}
\end{algorithm}

\subsection{Fast adaptation at the target node}

Based on the meta-backbone network $\phi$ and the meta-gating module $\theta$ transferred from the cloud, the target node $0$ is able to quickly learn a task-specific conditional channel gated network. Different from the simultaneous updates of both backbone networks and gating modules in federated meta-learning,  the gating module and the backbone network are updated once sequentially, at the target node 
by following a two-stage procedure using its local dataset $D_0$:

\textbf{Stage I:} Fix the task-specific backbone network as the meta-backbone network $\theta$, and update the task-specific gating module via one-step gradient descent  from the meta-gating module $\phi$ w.r.t. \eqref{localobject} using dataset $D_0$. Note that $\phi$ has effectively captured the important filters of a good meta-backbone network by leveraging the knowledge among different nodes in $\mathcal{S}$. Therefore, one single gradient update from $\phi$ by incorporating the local information is able to tune the gating module in a way that important channels for the specific task can be quickly selected, thus significantly reducing the  network size for updating.

\textbf{Stage II:} Given the adapted channel gating module $\tphi^0$, which determines the set of filters for local computation,  we next fine-tune the subnet via one-step gradient descent from the corresponding subnet in the meta-backbone model $\theta$, using the local dataset $D_0$. In this way, the training cost is further reduced as only a part of the backbone network gets involved in the single forward pass, even more efficient than fast-pruning methods, such as SNIP where at least one single forward pass needs to perform on the entire backbone network.  

Therefore, a task-specific conditional channel gated network can be quickly obtained at the target node for efficient inference.

\section{Theoretical Analysis}

In this section, we present the convergence analysis of Algorithm \ref{alg1} for a general non-convex local loss function.

First, a key observation here is that from the perspective of convergence analysis, the updates of the meta-backbone networks, i.e., vanilla gradient descent, in step $9$ of Algorithm \ref{alg1} are equivalent to the following proximal gradient descent:
\vspace{-0.1cm}
{\small
\begin{align*}
    \theta_t&=\prox_{\eta_t H}(\theta_{t-1}-\eta_t\nabla_{\theta,t}),\\
    \theta_t^{ag}&=\prox_{\beta_t H}(\theta_t^{pr}-\beta_t\nabla_{\theta,t}),
    \vspace{-0.1cm}
\end{align*}
}%
because $H$ is a function only of the meta-gating module $\phi$ and $\nabla_{\theta} H(\phi)=0$. Consequently, we can analyze the meta-backbone network $\theta$ and the meta-gating module $\phi$ together, and examine the convergence of Algorithm \ref{alg1} w.r.t. $w=(\phi,\theta)$. Let $w_t^{pr}=(\phi_t^{pr},\theta_t^{pr})$, $w_t^{ag}=(\phi_t^{ag},\theta_t^{ag})$ and $\tw_t^i=(\tphi_t^i,\ttheta_t^i)$. The step $9$ in Algorithm \ref{alg1} is then equivalent to the following:
\vspace{-0.1cm}
{\small
\begin{align}
    w_t=&\prox_{\eta_t H} (w_{t-1}-\eta_t \nabla_{w,t}),\\
    w_t^{ag}=&\prox_{\beta_t H}(w_t^{pr}-\beta_t \nabla_{w,t}),
    \vspace{-0.1cm}
\end{align}
}%
where $\nabla_{w,t}=\frac{\lambda}{N}\sum_{i=1}^N(w_t^{pr}-\tw_t^i)$.

We next characterize the structural properties of the federated meta-learning objective $F(w)$. For ease of exposition, let $G(w)=\frac{1}{N}\sum_{i=1}^N \min_{\tw^i} G_i(\tw^i;w)$.
As is standard, we make the following assumptions.
\begin{assumption}\label{assum1}
The loss function $L_i$ is twice-differentiable and $\rho$-smooth, i.e., $\|\nabla L_i(w)-\nabla L_i(w')\|\leq \rho\|w-w'\|$.
\end{assumption}
\begin{assumption}\label{assum2}
$H(\cdot)$ is a proper closed convex function, and $\|w\|\leq M$. This implies that $\|\prox_{cH}(w-cg)\|\leq M$ for any $c>0$ and $g$.
\end{assumption}

It can be shown that Assumption \ref{assum2} immediately holds for $\mathcal{L}_1$-norm and Group Lasso with bounded domain \cite{ghadimi2016accelerated}.
Following the same line as in \cite{zhou2019efficient}, we have the following lemma:
\begin{lemma}\label{lem:smooth}
Suppose that Assumption \ref{assum1} holds. For $\lambda>\rho$, $G(w)$ is $\frac{\lambda \rho}{\lambda+\rho}$-smooth w.r.t. $w$.
\end{lemma}

Therefore, the federated meta-learning problem \eqref{metaobject} can be rewritten as 
\vspace{-0.1cm}
{\small
\begin{align}\label{rewrite}
    \min_w~~F(w)=H(w)+G(w)
    \vspace{-0.2cm}
\end{align}
}%
where $H(\cdot)$ is convex and non-smooth, and $G(\cdot)$ is non-convex and smooth.
We aim to establish the convergence of Algorithm \ref{alg1} for finding a first-order stationary point of $F(w)$, which is defined as follows.
\begin{definition}
$w$ is a first-order stationary point of $F$ if $0\in \partial H(w)+\nabla G(w)$, where $\partial H(\cdot)$ denotes the subdifferential of $H(\cdot)$.
\end{definition}

Let $\mathcal{Q}(w,g,c)=\frac{1}{c}[w-\prox_{cH}(w-cg)]$. When $g=\nabla G(w)$, $\mathcal{Q}(w,g,c)$ is generally called the gradient mapping at $w$ \cite{parikh2014proximal}. It is easy to tell that $\mathcal{Q}(w,\nabla G(w),c)=\nabla G(w)$ if $H(w)=0$. Hence, the value of $\mathcal{Q}(w,g,c)$ is often used as a termination criterion for solving non-smooth optimization problem as a surrogate of the subdifferentials, based on the following lemma \cite{gong2013general,attouch2013convergence}:
\begin{lemma}\label{lem:opti}
$w$ is a first-order stationary point of $F$ if and only if $\mathcal{Q}(w,g,c)=0$.
\end{lemma}

It has been shown in \cite{ghadimi2016accelerated} that $\|\mathcal{Q}(w,g,c)\|^2$ can be used to quantify the gap between $w$ and the first-order stationary point of $F$, in the sense that this optimality gap converges to $0$ as the value of $\|\mathcal{Q}(w,g,c)\|^2$ vanishes. Therefore, in this work we seek to establish the convergence of Algorithm \ref{alg1} in terms of an $\epsilon$-first order stationary point, i.e.,
$\|\mathcal{Q}(w,g,c)\|^2\leq\epsilon$.

As mentioned earlier, the global update direction $\nabla_{w,t}$ is indeed an approximate gradient w.r.t. to function $G$. To show the convergence, we first evaluate this gradient estimation gap, which is captured by the following lemma.
\begin{restatable}{lemma}{lemmaa}\label{lem:boundgap}
Suppose that Assumption \ref{assum1} holds. For $\lambda>\rho$, the following equality holds:
\begin{equation*}
    \nabla_{w,t}=\nabla_{w} G(w_t^{pr})+\delta_{t},
\end{equation*}
where $\|\delta_{t}\|^2\leq\frac{\lambda^2\xi_t}{(\lambda-\rho)^2}$.
\end{restatable}

For convenience, denote an auxiliary sequence
{\small
\begin{align*}
    \Gamma_t=\begin{cases}
    1,~ &t=1,\\
    (1-\alpha_t)\Gamma_{t-1},~ &t\geq 2.
    \end{cases}
\end{align*}
}%
Let $w^*$ be the optimal solution to problem \eqref{rewrite}, i.e., $F(w^*)=\min_w F(w)$. 
We can have the following main theorem about the convergence of Algorithm \ref{alg1}.

\begin{restatable}{theorem}{theorema}\label{theorem1}
Suppose that Assumptions \ref{assum1} and \ref{assum2} hold. For any $t\geq 1$, let $\alpha_t=\frac{2}{t+1}$, $\beta_t<\frac{\lambda+\rho}{\lambda \rho}$ and $\eta_t$ satisfy:
{\small
\begin{align*}
    \alpha_t\eta_t\leq \beta_t, ~  \frac{\alpha_t}{\Gamma_t}\left(\frac{1}{\eta_t}-1\right)\geq \frac{\alpha_{t+1}}{\Gamma_{t+1}}\left(\frac{1}{\eta_{t+1}}-1\right).
\end{align*}
}%
Then, for $\lambda>\rho$, we have
{\small
\begin{align*}
    \min_{t\in[1,T]}& \|\mathcal{Q}(w_t^{pr},\nabla_w G(w_t^{pr}),\beta_t)\|^2
    \leq\frac{24\lambda^3 \rho\sum_{t=1}^T
   \xi_t}{T(\lambda+\rho)(\lambda-\rho)^2}\\
    &+\frac{24(\lambda \rho)^2(\|w^*\|^2+2M)}{T(\lambda+\rho)^2}
    +\frac{C\|w_0-w^*\|^2}{T^2(T+1)}
\end{align*}
}%
where $C=\frac{24\lambda \rho}{(\lambda+\rho)}\left(\frac{1}{\eta_1}-1\right)$.
\end{restatable}

It can be seen from Theorem \ref{theorem1} that the convergence error consists of three terms, where the first term captures the impact of inexact solutions to each local learning problem \eqref{localmin}, and the last two terms quickly converge to $0$ in the faster rate of $O(\frac{1}{T})$, compared with the rate of $O(\frac{1}{\sqrt{T}})$ in \cite{zhou2019efficient}. Let $\psi=\frac{24\lambda^3 \rho\sum_{t=1}^T \xi_t}{T(\lambda+\rho)(\lambda-\rho)^2}$. Clearly, if the accumulated estimation error $\sum_{t=1}^T \xi_t=o(T)$ for local learning problem \eqref{localmin}, $\psi$ will vanish eventually. Therefore, Algorithm \ref{alg1} can find an $O(\epsilon+\psi)$-first order stationary point in at most $O(\epsilon^{-1})$ communication rounds between the cloud and nodes.


\section{Experiments}

To evaluate the performance of the proposed framework MetaGater, we seek to answer the following questions in this section: (1) What is the performance of the proposed federated meta-learning approach when compared with other federated meta-learning methods? (2) What is the impact of the channel gating modules on the fast adaptation performance at the target node? (3) What is the performance of MetaGater when compared with other fast pruning methods?

\subsection{Experimental setup}

\textbf{Datasets.} In the experiments, we first study the  image classification problem as the learning tasks on two widely used datasets, MNIST \cite{lecun1998gradient} and CIFAR-10 \cite{krizhevsky2009learning}. In particular, for MNIST, we distribute the data among $N=20$ tasks, such that the number of samples per task is in the range of [1165, 3834]. We randomly select 5 tasks
for each round in federated meta-training, and 5 target tasks for fast adaptation. For CIFAR-10, the data is distributed among $N=50$ tasks where the number of samples per task is in the range of [221, 2792]. Similarly, we randomly select 20 tasks
for each round in federated meta-training, and 20 target tasks for fast adaptation.
Also, each task only has data samples from two classes for both datasets. (Due to the space limitation, we relegate more experimental results to the appendix.)


\textbf{Models.} For MNIST, we consider a  two-layer neural network with a hidden layer of size 100.  For CIFAR-10, we consider a four-layer convolutional neural network with three convolutional layers, followed by a fully connected layer. The layer sizes are 32, 64, 128 and 2048, respectively. To minimize the introduced model overhead, we integrate the channel gating module at  the third convolutional layer, where Group Lasso is used to prompt sparsity of the channel gating module. It is worth noting that the same channel gating technique can be applied to other layers along the same line.

\textbf{Parameter setup.} As is standard in federated learning \cite{mcmahan2017communication}, 
we evaluate the performance under a fixed number of communication rounds, with $T=800$ and $T=400$ for MINIST and CIFAR-10, respectively. Other hyper-parameters are same for both datasets. For $t$-th round,
the learning rate $\alpha_t=\frac{2}{t+1}$, and we choose $\beta_t=\alpha_t\eta_t=1$. Besides, $\lambda=0.2$. During federated meta-learning, we run gradient descent for multiple local updates to solve the local minimization problem \eqref{localmin} for each training tasks.  For fast adaptation at target tasks, we only run one-step gradient descent to fine-tune both meta-backbone network and meta-gating module. We 
evaluate the testing accuracy at the target tasks,
and repeat all the experiments for 5 times to obtain the average performance.

\subsection{Meta-models via federated meta-learning}

To answer the first question, we consider two existing baseline algorithms, i.e., the classical federated learning algorithm FedAvg \cite{mcmahan2017communication} and one state-of-the-art federated meta-learning approach Per-FedAvg \cite{fallah2020personalized}. For a fair comparison, we first remove the channel gating module, consider meta-learning of the backbone network, and also update the output of FedAvg with one-step gradient descent as in Per-FedAvg for testing at the target task. Note that we do not use pFedMe \cite{t2020personalized} as a baseline for the following reasons: (1) The performance of the global model therein is worse than Per-FedAvg; (2) although the personalized models can achieve better accuracy, it requires to test all personalized models and pick the best one, which is clearly not suitable for on-device learning.

As illustrated in Fig.~\ref{Fig:convergence}, MetaGater clearly converges faster than FedAvg and Per-FedAvg, which is  important in federated learning as the communication cost usually is a bottleneck in wireless networks. Moreover, we  have the following observations based on Table \ref{tab:fedmeta}:
(1) MetaGater achieves the best accuracy performance among all the methods. (2)
For learning the meta-backbone network only, MetaGater achieves a much higher testing accuracy with less training time compared with Per-FedAvg, even with only one-step gradient update locally during the training process. Note that the longer training time for Per-FedAvg is because it need to evaluate the local gradient for two times in one local update, in order to approximate the gradient w.r.t. the meta-model.
Such a performance improvement firmly corroborates the benefits of utilizing higher-order information of the meta-objective function through proximal updates and accelerating the global meta-model updates with momentum. (3) As expected, it takes longer to  jointly train the meta-backbone network and the meta-gating module for MetaGater with gating, compared to training MetaGater without  gating. This extra training time is nevertheless not important from an offline training perspective. More importantly,  the \emph{structural} sparsity offered by the gating module can help to  achieve the agile adaptability at different new tasks.

\begin{figure}
\centering
\includegraphics[scale=0.13]{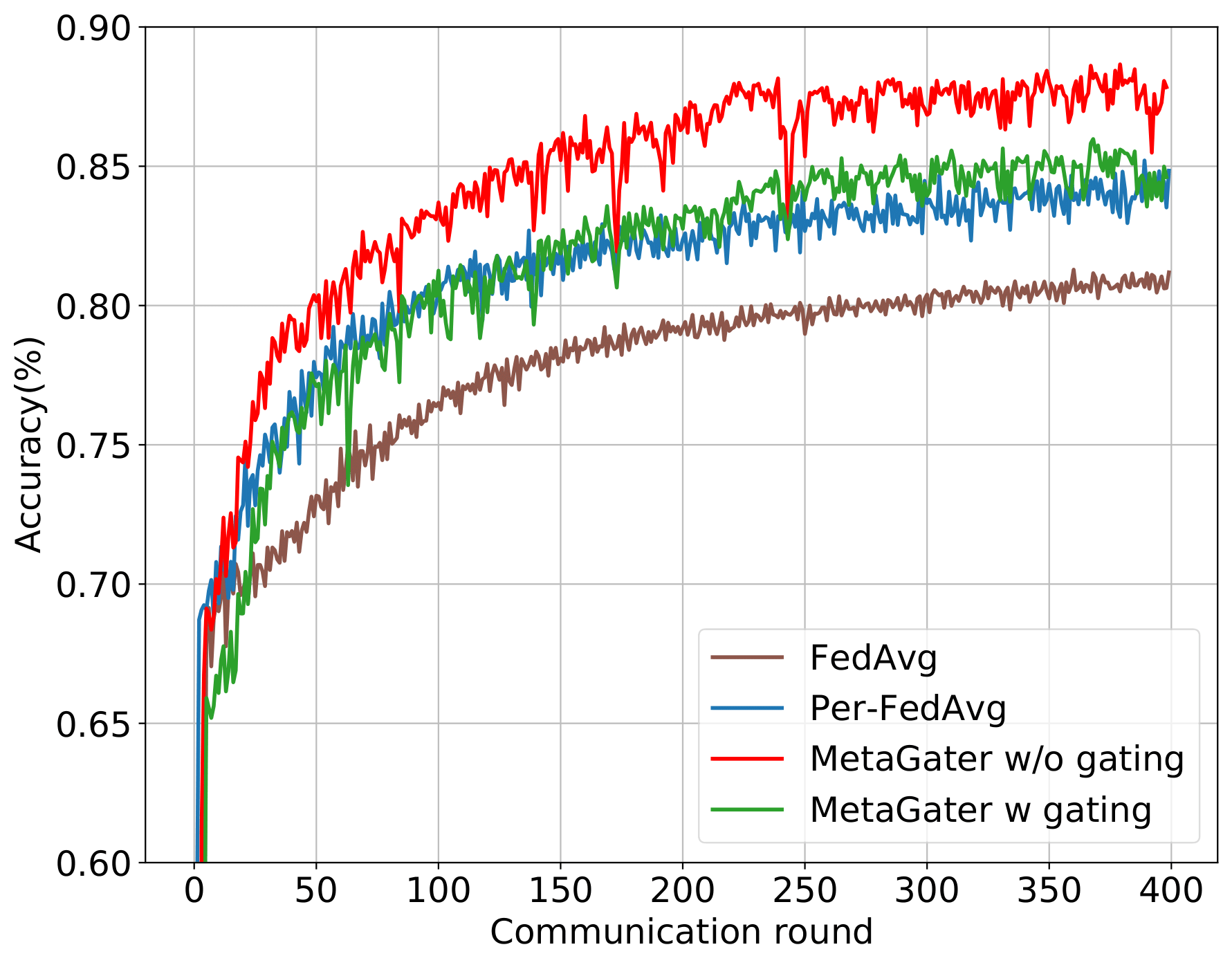}
\caption{Convergence behavior of MetaGater, Per-FedAvg, FedAvg  on CIFAR-10. We compare testing accuracy for target tasks.}
\label{Fig:convergence}
\end{figure}

\begin{table}
\small
    \begin{center}
    \begin{tabular}{c|cccc}
         \toprule
         Dataset &
          Method &
          \begin{tabular}[c]{@{}l@{}}
               Local  \\
               update 
          \end{tabular} & Accuracy(\%) & \begin{tabular}[c]{@{}l@{}}
               Training  \\
               time(s)
          \end{tabular}\\
          \hline
          \multirow{6}{*}{MNIST}&
         \multirow{2}{*}{FedAvg}
            & 1 & $98.78\pm0.03$ & 396\\
          &  & 2 & - & -\\
         \cline{2-5}
         &
         \multirow{2}{*}{Per-FedAvg}
            & 1 & $99.05\pm0.05$ & 920\\
         &   & 2 & - & -\\
        \cline{2-5}
        &
         \multirow{2}{*}{
         \begin{tabular}[c]{@{}l@{}}
               MetaGater  \\
               w/o gating
          \end{tabular}}
            & 1 & $99.2\pm0.05$ & 763\\
          &  & 2 & $99.4\pm0.03$ & 1417\\
            \hline
            \multirow{6}{*}{CIFAR10}&
         \multirow{2}{*}{FedAvg}
            & 1 & $80.9\pm1.2$ & 1139\\
          &  & 2 & - & -\\
         \cline{2-5}
         &
         \multirow{2}{*}{Per-FedAvg}
            & 1 & $82.5\pm2.3$ & 2056\\
         &   & 2 & - & -\\
        \cline{2-5}
        &
         \multirow{2}{*}{
         \begin{tabular}[c]{@{}l@{}}
               MetaGater  \\
               w/o gating
          \end{tabular}}
            & 1 & $88.1\pm2.0$ & 1186\\
          &  & 2 & $88.6\pm2.5$ & 2046\\
                    \cline{2-5}
         &
          \multirow{2}{*}{
          \begin{tabular}[c]{@{}l@{}}
               MetaGater  \\
               w/ gating
           \end{tabular}}
             & 1 & $87.5\pm2.6$ & 1960\\
           &  & 2 & $88.2\pm2.8$ & 3719\\
            \hline
    \end{tabular}
    \end{center}
    \caption{Accuracy comparison for MetaGater, Per-FedAvg, FedAvg on MNIST and CIFAR-10. Clearly, MetaGater achieves the best accuracy among all methods. }
    \label{tab:fedmeta}
    \vspace{-0.5cm}
\end{table}

\subsection{Impact of channel gating module}

To highlight the impact of channel gating module, we compare the fast adaptation performance at the target tasks between (1) MetaGater without (w/o) channel gating module and (2) MetaGater with (w/) channel gating module. 
It can be seen from Table \ref{tab:fedmeta} that at the cost of  extra training time  to jointly meta-train the backbone network and the gating module, MetaGater can achieve the agile adaptability at different target tasks. In particular, 
for the fast adaptation performance shown in Table~\ref{tab:gating}, with channel gating module, the target task is able to quickly obtain a more compact model for efficient inference with only a slightly degradation in the accuracy, compared with MetaGater w/o gating module. 

\begin{table}
\small
    \begin{center}
    \begin{tabular}{cccc}
         \toprule
          Method &
          Accuracy(\%) & Sparsity(\%) & Learning time(s)\\
          \midrule
        MetaGater & \multirow{2}{*}{$88.1\pm2.0$} &\multirow{2}{*}{-}
        &\multirow{2}{*}{1.2}\\
        w/o gating & & \\
        \midrule
         MetaSNIP
            & $86.8\pm3.3$ & $25\pm1.2$ & 1.4\\
        \midrule
          MetaGater & \multirow{2}{*}{$87.5\pm2.6$} &\multirow{2}{*}{$25\pm3.2$}
        &\multirow{2}{*}{1.03}\\
        w/ gating & & \\
            \bottomrule
    \end{tabular}
    \end{center}
    \caption{Fast adaptation performance comparison on CIFAR-10. The sparsity represents the \emph{channel-wise sparsity} in the third convolutional layer. 
    Compared with MetaSNIP, MetaGater has a better accuracy and a larger sparsity range. The learning time is the total time for fast adaptation and inference at the target tasks.}
    \label{tab:gating}
     \vspace{-0.5cm}
\end{table}

\subsection{Performance of MetaGater}

Since this study focuses on the fast learning performance based on distributed learning, most existing methods based on centralized pre-training on a large target dataset cannot directly serve as a baseline without nontrivial modification. Note that SNIP \cite{lee2018snip} is a fast pruning method which directly pruns the initial network at single-shot, and hence eliminates the need of pre-training.
To fairly evaluate the performance of MetaGater, we compare  MetaGater with the following approach (referred as MetaSNIP): we train a meta-backbone network using MetaGater without gating module, apply SNIP to quickly obtain a sparse backbone network, and then fine-tune it using one-step gradient descent.

It can be seen from Table \ref{tab:gating} that MetaGater clearly outperforms MetaSNIP in the following sense:
(1) MetaGater is able to obtain a better subnet with higher accuracy, in a similar speed with MetaSNIP; (2) After fast adaptation, MetaGater exhibits a larger diversity of the achieved model sparsity on different target tasks. This larger diversity implies a better sensitivity of the learnt meta-gating module w.r.t. one-step gradient descent, which enables the fast learning of the task-specific channel gated network.

\section{Conclusion}

In this work, we  propose MetaGater, a fast learning framework of conditional channel gated networks via federated meta-learning, where good meta-initializations for both backbone networks and gating modules are jointly learnt by leveraging the model similarity across learning tasks on different nodes. As the meta-gating module effectively captures the important filters of a good meta-backbone network and offers structural sparsity across tasks, it can achieve the agile adaptability at different new tasks by quickly learning a task-specific conditional channel gated network. Particularly, to efficiently solve the federated meta-learning problem, we further propose a novel approach based on accelerated proximal gradient descent with inexact solutions to the local problem, and show that an $\epsilon$-first order stationary point can be obtained in at most $O(\epsilon^{-1})$ communication rounds for non-convex functions. Experiments showcase that our federated meta-learning approach clearly outperforms other existing methods, and MetaGater quickly returns an task-specific channel gated networks after fast adaptation, with a higher accuracy than one state-of-the-art method in fast pruning.

{\small
\bibliographystyle{ieee_fullname}
\bibliography{egbib}
}

\onecolumn
\section*{Appendix}
\appendix

\section{Auxiliary Lemmas}

To prove Theorem \ref{theorem1}, we first restate some known results to be used later.

\begin{lemma}\label{lemmaaux3}\cite{zhou2019efficient}
Assume that $g(u)$ is $\lambda$-strongly convex. Then we have
\begin{align*}
    \langle \nabla g(u), u-u^*\rangle\geq \lambda \|u-u^*\|^2
\end{align*}
where $u^*=\arg\min g(u)$.
\end{lemma}

\begin{lemma}\label{lemmaaux1}\cite{ghadimi2016accelerated}
For a sequence of learning rates $\{\alpha_t\}$ and the sequence $\{\gamma_t\}$ satisfies
\begin{align*}
    \gamma_t\leq (1-\alpha_t)\gamma_{t-1}+\tau_t,~t=1, 2, ...,
\end{align*}
 we have that $\gamma_t\leq \Gamma_t\sum_{i=1}^t(\tau_i/\Gamma_i)$ for any $t\geq 1$
 where
\begin{align*}
    \Gamma_t=\begin{cases}
    1, & t=1,\\
    (1-\alpha_t)\Gamma_{t-1}, & t\geq 2.
    \end{cases}
\end{align*}

\end{lemma}

The next lemma characterizes the solution of the proximal updates in step 9 of Algorithm \ref{alg1}.

\begin{lemma}\label{lemmaaux2}\cite{ghadimi2012optimal}
 Let $h: X\rightarrow \mathbb{R}$ be a differentiable convex function and $V(x,z)$ be defined as 
\begin{align*}
    V(x,z)=h(z)-[h(x)+\langle \nabla h(x), z-x \rangle].
\end{align*}
For a given convex function $p: X\rightarrow \mathbb{R}$, the points $\Tilde{x},\Tilde{y}\in X$ and the scalars $\mu_1,\mu_2\geq 0$,
if
\begin{align*}
    u^*\in \arg\min \{p(u)+\mu_1 V(\Tilde{x},u)+\mu_2 V(\Tilde{y},u): u\in X\},
\end{align*}
then for any $u\in X$, we have
\begin{align*}
    p(u^*)+\mu_1 V(\Tilde{x},u^*)+\mu_2 V(\Tilde{y},u^*)\leq p(u)+\mu_1 V(\Tilde{x},u)+\mu_2 V(\Tilde{y},u)-(\mu_1+\mu_2)V(u^*,u).
\end{align*}
\end{lemma}

To show the convergence of the algorithm, we next need to connect the proximal update, based on the surrogate gradient $\nabla_{wt}$, with the true gradient $\nabla_w G(w_t^{pr})$. 

\begin{restatable}{lemma}{lemmab}\label{lem:boundproximal}
The following inequalities hold for any $w$:
\begin{align*}
    \langle\nabla_w G(w_t^{pr}),w_t-w\rangle
    \leq& H(w)-H(w_t)+\left(\frac{1}{2\eta_t}-\frac{1}{2}\right)[\|w_{t-1}-w\|^2-\|w_t-w\|^2]\\
    &-\frac{1}{2\eta_t}\|w_t-w_{t-1}\|^2+\frac{1}{2}\|\delta_t\|^2,\\
    \langle\nabla_w G(w_t^{pr}),w_t^{ag}-w\rangle
    \leq& H(w)-H(w_t^{ag})+\frac{1}{2\beta_t}[\|w_t^{pr}-w\|^2-\|w_t^{ag}-w_{t}^{pr}\|^2]+\frac{1}{2}\|\delta_t\|^2.
\end{align*}
\end{restatable}

\section{Proof of Lemma \ref{lem:boundgap}}

\lemmaa*

\begin{proof}
Let $\tw^{i*}_t=(\tphi_t^{i*},\ttheta_t^{i*})$ denote the optimal solution to the local optimization problem \eqref{localobject}. We first investigate the structural property of the meta-learning objective function. 
Since $L_i$ is $\rho$-smooth, for $\lambda>\rho$, we can know that $G$ is $\frac{\lambda \rho}{\lambda+\rho}$-smooth based on Lemma \ref{lem:smooth}. Then, based on the chain rule, we can know 
\begin{align*}
    \nabla_w G(w_t^{pr})&=\frac{1}{N}\sum_{i=1}^N\left\{\left(\frac{\partial \tw_t^{i*}}{\partial w}\right)^T\bigg|_{w=w_t^{pr}}\nabla_w L_i(\tw_t^{i*})+\lambda\left(I-\left(\frac{\partial \tw_t^{i*}}{\partial w}\right)^T\bigg|_{w=w_t^{pr}}\right)(w_t^{pr}-\tw_t^{i*})\right\}\\
    &=\frac{1}{N}\sum_{i=1}^N\left\{\lambda (w_t^{pr}-\tw_t^{i*})+\left(\frac{\partial \tw_t^{i*}}{\partial w}\right)^T\bigg|_{w=w_t^{pr}}[\nabla_w L_i(\tw_t^{i*})-\lambda(w_t^{pr}-\tw_t^{i*})]\right\}\\
    &=\frac{\lambda}{N}\sum_{i=1}^N (w_t^{pr}-\tw_t^{i*})
\end{align*}
where the last equality holds because the following first-order optimality condition for $\tw_t^{i*}$:
\begin{align*}
    \nabla_w L_i(\tw_t^{i*})+\lambda(\tw_t^{i*}-w_t^{pr})=0.
\end{align*}
Therefore, we can have that
\begin{align*}
    \nabla_{w,t}=&\frac{\lambda}{N}\sum_{i=1}^N (w_t^{pr}-\tw_t^i)\\
    =&\frac{\lambda}{N}\sum_{i=1}^N(w_t^{pr}-\tw_t^{i*}+\tw_t^{i*}-\tw_t^i)\\
    =& \nabla_w G(w_t^{pr})+\frac{\lambda}{N}\sum_{i=1}^N(\tw_t^{i*}-\tw_t^i)\\
    =& \nabla_w G(w_t^{pr})+\delta_t.
\end{align*}

Moreover, it is clear that $G_i$ is $(\lambda-\rho)$-strongly convex. Based on Lemma \ref{lemmaaux3}, it follows that
\begin{align*}
    \|\tw^{i*}_t-\tw^i_t\|^2\leq \frac{1}{(\lambda-\rho)^2}\|\nabla_w G_i(\tw^i_t)\|^2\leq \frac{\xi_t}{(\lambda-\rho)^2}.
\end{align*}
Hence, we can have
\begin{align*}
    \|\delta_t\|^2=&\lambda^2\|\frac{1}{N}\sum_{i=1}^N(\tw_t^{i*}-\tw_t^i)\|^2\\
    \leq&\frac{\lambda^2}{N}\sum_{i=1}^N \|\tw_t^{i*}-\tw_t^i\|^2\\
    \leq& \frac{\lambda^2 \xi_t}{(\lambda-\rho)^2},
\end{align*}
thereby completing the proof of Lemma \ref{lem:boundgap}.

\end{proof}

\section{Proof of Lemma \ref{lem:boundproximal}}

\lemmab*

\begin{proof}
From the definition of the proximal update for the solution $w_t$, and Lemma \ref{lemmaaux2}, we can obtain that for any $w$
\begin{align}
    \langle\nabla_{w,t}, w_t-w\rangle+H(w_t)\leq H(w)+\frac{1}{2\eta_t}[\|w_{t-1}-w\|^2-\|w_t-w\|^2-\|w_t-w_{t-1}\|^2].
\end{align}
Based on Lemma \ref{lem:boundgap}, it follows that
\begin{align*}
    &\langle \nabla_w G(w_t^{pr})+\delta_t, w_t-w\rangle+H(w_t)\\
    =& \langle\nabla_w G(w_t^{pr}),w_t-w\rangle+\langle\delta_t,w_t-w \rangle\\
    \leq& H(w)+\frac{1}{2\eta_t}[\|w_{t-1}-w\|^2-\|w_t-w\|^2-\|w_t-w_{t-1}\|^2],
\end{align*}
such that 
\begin{align*}
    &\langle\nabla_w G(w_t^{pr}),w_t-w\rangle\\
    \leq& H(w)-H(w_t)+\frac{1}{2\eta_t}[\|w_{t-1}-w\|^2-\|w_t-w\|^2-\|w_t-w_{t-1}\|^2]-\langle\delta_t,w_t-w \rangle\\
    \leq&H(w)-H(w_t)+\frac{1}{2\eta_t}[\|w_{t-1}-w\|^2-\|w_t-w\|^2-\|w_t-w_{t-1}\|^2]+\frac{1}{2}\|\delta_t\|^2+\frac{1}{2}\|w_t-w\|^2\\
    \leq&H(w)-H(w_t)+\left(\frac{1}{2\eta_t}-\frac{1}{2}\right)[\|w_{t-1}-w\|^2-\|w_t-w\|^2]-\frac{1}{2\eta_t}\|w_t-w_{t-1}\|^2+\frac{1}{2}\|\delta_t\|^2.
\end{align*}
Similarly, for the solution $w_t^{ag}$, we have
\begin{align}
    \langle\nabla_{w,t}, w_t^{ag}-w\rangle+H(w_t^{ag})\leq H(w)+\frac{1}{2\beta_t}[\|w_t^{pr}-w\|^2-\|w_t^{ag}-w\|^2-\|w_t^{ag}-w_{t}^{pr}\|^2].
\end{align}
Therefore, it follows that 
\begin{align*}
    &\langle\nabla_w G(w_t^{pr}),w_t^{ag}-w\rangle\\
    \leq& H(w)-H(w_t^{ag})+\frac{1}{2\beta_t}[\|w_t^{pr}-w\|^2-\|w_t^{ag}-w\|^2-\|w_t^{ag}-w_{t}^{pr}\|^2]+\frac{1}{2}\|\delta_t\|^2+\frac{1}{2}\|w_t^{ag}-w\|^2\\
    =& H(w)-H(w_t^{ag})+\frac{1}{2\beta_t}[\|w_t^{pr}-w\|^2-\|w_t^{ag}-w_{t}^{pr}\|^2]+\frac{1}{2}\|\delta_t\|^2-\left(\frac{1}{2\beta_t}-\frac{1}{2}\right)\|w_t^{ag}-w\|^2\\
    \leq& H(w)-H(w_t^{ag})+\frac{1}{2\beta_t}[\|w_t^{pr}-w\|^2-\|w_t^{ag}-w_{t}^{pr}\|^2]+\frac{1}{2}\|\delta_t\|^2
\end{align*}
for $\beta_t\leq1$.
\end{proof}

\section{Proof of Theorem \ref{theorem1}}

\theorema*

\begin{proof}
Based on the smoothness of $G$, we can conclude that
\begin{align}\label{smooth1}
    G(w_t^{ag})\leq& G(w_t^{pr})+\langle\nabla_w G(w_t^{pr}), w_t^{ag}-w_t^{pr}\rangle+\frac{\lambda \rho}{2(\lambda+\rho)}\|w_t^{ag}-w_t^{pr}\|^2,
\end{align}
\begin{align}\label{smooth2}
    G(w_t^{pr})\leq& G(w)+\langle \nabla_w G(w_t^{pr}), w_t^{pr}-w\rangle+\frac{\lambda \rho}{2(\lambda+\rho)}\|w-w_t^{pr}\|^2,
\end{align}
\begin{align}\label{smooth3}
    G(w_t^{pr})\leq& G(w_{t-1}^{ag})+\langle \nabla_w G(w_t^{pr}), w_t^{pr}-w_{t-1}^{ag}\rangle+\frac{\lambda \rho}{2(\lambda+\rho)}\|w_{t-1}^{ag}-w_t^{pr}\|^2.
\end{align}
It follows that
\begin{align}\label{decompose}
    &G(w_t^{pr})-[(1-\alpha_t)G(w_{t-1}^{ag})+\alpha_t G(w)]\nonumber\\
    =&(1-\alpha_t)[G(w_t^{pr})-G(w_{t-1}^{ag})]+\alpha_t[G(w_t^{pr})-G(w)]\nonumber\\
    \leq& (1-\alpha_t)\left[\langle \nabla_w G(w_t^{pr}), w_t^{pr}-w_{t-1}^{ag}\rangle+\frac{\lambda \rho}{2(\lambda+\rho)}\|w_{t-1}^{ag}-w_t^{pr}\|^2\right]\nonumber\\
    &+\alpha_t\left[\langle \nabla_w G(w_t^{pr}), w_t^{pr}-w\rangle+\frac{\lambda \rho}{2(\lambda+\rho)}\|w-w_t^{pr}\|^2\right]\nonumber\\
    =&\langle\nabla_w G(w_t^{pr}), w_t^{pr}-[(1-\alpha_t)w_{t-1}^{ag}+\alpha_t w]\rangle+\frac{\lambda \rho(1-\alpha_t)}{2(\lambda+\rho)}\|w_{t-1}^{ag}-w_t^{pr}\|^2+\frac{\lambda \rho\alpha_t}{2(\lambda+\rho)}\|w-w_t^{pr}\|^2\nonumber\\
    =&\langle\nabla_w G(w_t^{pr}), w_t^{pr}-[(1-\alpha_t)w_{t-1}^{ag}+\alpha_t w]\rangle+\frac{\lambda \rho(1-\alpha_t)\alpha_t^2}{2(\lambda+\rho)}\|w_{t-1}^{ag}-w_{t-1}\|^2+\frac{\lambda \rho\alpha_t}{2(\lambda+\rho)}\|w-w_t^{pr}\|^2
\end{align}
where the last equality is true because
\begin{align*}
    \|w_{t-1}^{ag}-w_t^{pr}\|=&\|w_{t-1}^{ag}-\alpha_t w_{t-1}-(1-\alpha_t)w_{t-1}^{ag}\|\\
    =&\alpha_t\|w_{t-1}^{ag}-w_{t-1}\|.
\end{align*}

Combining \eqref{smooth1} and \eqref{decompose}, we can have that
\begin{align}\label{firsttwo}
    G(w_t^{ag})\leq& (1-\alpha_t)G(w_{t-1}^{ag})+\alpha_tG(w)+\langle\nabla_w G(w_t^{pr}), w_t^{ag}-w_t^{pr}\rangle+\frac{\lambda \rho}{2(\lambda+\rho)}\|w_t^{ag}-w_t^{pr}\|^2\nonumber\\
    &+\langle\nabla_w G(w_t^{pr}), w_t^{pr}-[(1-\alpha_t)w_{t-1}^{ag}+\alpha_t w]\rangle+\frac{\lambda \rho(1-\alpha_t)\alpha_t^2}{2(\lambda+\rho)}\|w_{t-1}^{ag}-w_{t-1}\|^2+\frac{\lambda \rho\alpha_t}{2(\lambda+\rho)}\|w-w_t^{pr}\|^2\nonumber\\
    =&(1-\alpha_t)G(w_{t-1}^{ag})+\alpha_tG(w)+\langle\nabla_w G(w_t^{pr}),w_t^{ag}-[(1-\alpha_t)w_{t-1}^{ag}+\alpha_t w]\rangle\nonumber\\
    &+\frac{\lambda \rho}{2(\lambda+\rho)}\|w_t^{ag}-w_t^{pr}\|^2+\frac{\lambda \rho(1-\alpha_t)\alpha_t^2}{2(\lambda+\rho)}\|w_{t-1}^{ag}-w_{t-1}\|^2+\frac{\lambda \rho\alpha_t}{2(\lambda+\rho)}\|w-w_t^{pr}\|^2.
\end{align}

Next, we need to bound the term $\langle\nabla_w G(w_t^{pr}),w_t^{ag}-[(1-\alpha_t)w_{t-1}^{ag}+\alpha_t w]\rangle$.
Based on Lemma \ref{lem:boundproximal}, it follows that
\begin{align*}
    &\langle\nabla_w G(w_t^{pr}), w_t^{ag}-[(1-\alpha_t)w_{t-1}^{ag}+\alpha_t w]\rangle\\
  =&\langle\nabla_w G(w_t^{pr}), w_t^{ag}-(1-\alpha_t)w_{t-1}^{ag}-\alpha_t w_t\rangle+\alpha_t\langle\nabla_w G(w_t^{pr}),w_t-w\rangle\\
  \leq& H((1-\alpha_t)w_{t-1}^{ag}+\alpha_t w_t)-H(w_t^{ag})+\frac{1}{2\beta_t}[\|w_t^{pr}-(1-\alpha_t)w_{t-1}^{ag}-\alpha_t w_t\|^2-\|w_t^{ag}-w_t^{pr}\|^2]+\frac{1}{2}\|\delta_t\|^2\\
  &+\alpha_t H(w)-\alpha_t H(w_t)+\alpha_t\left(\frac{1}{2\eta_t}-\frac{1}{2}\right)[\|w_{t-1}-w\|^2-\|w_t-w\|^2]
  -\frac{\alpha_t}{2\eta_t}\|w_t-w_{t-1}\|^2+\frac{\alpha_t}{2}\|\delta_t\|^2\\
  \overset{(a)}{\leq}&(1-\alpha_t)H(w_{t-1}^{ag})+\alpha_tH(w_t)-H(w_t^{ag})+\frac{1}{2\beta_t}[\|w_t^{pr}-(1-\alpha_t)w_{t-1}^{ag}-\alpha_t w_t\|^2-\|w_t^{ag}-w_t^{pr}\|^2]+\frac{1}{2}\|\delta_t\|^2\\
  &+\alpha_t H(w)-\alpha_t H(w_t)+\alpha_t\left(\frac{1}{2\eta_t}-\frac{1}{2}\right)[\|w_{t-1}-w\|^2-\|w_t-w\|^2]
  -\frac{\alpha_t}{2\eta_t}\|w_t-w_{t-1}\|^2+\frac{\alpha_t}{2}\|\delta_t\|^2\\
  \overset{(b)}{\leq}&(1-\alpha_t)H(w_{t-1}^{ag})+\alpha_tH(w)-H(w_t^{ag})+\frac{1}{2\beta_t}[\alpha_t^2\|w_t-w_{t-1}\|^2-\|w_t^{ag}-w_t^{pr}\|^2]+\frac{1+\alpha_t}{2}\|\delta_t\|^2\\
  &+\alpha_t\left(\frac{1}{2\eta_t}-\frac{1}{2}\right)[\|w_{t-1}-w\|^2-\|w_t-w\|^2]
  -\frac{\alpha_t}{2\eta_t}\|w_t-w_{t-1}\|^2\\
  =&(1-\alpha_t)H(w_{t-1}^{ag})+\alpha_tH(w)-H(w_t^{ag})+\alpha_t\left(\frac{1}{2\eta_t}-\frac{1}{2}\right)[\|w_{t-1}-w\|^2-\|w_t-w\|^2]\\
  &+\left(\frac{\alpha_t^2}{2\beta_t}-\frac{\alpha_t}{2\eta_t}\right)\|w_t-w_{t-1}\|^2+\frac{1+\alpha_t}{2}\|\delta_t\|^2-\frac{1}{2\beta_t}\|w_t^{ag}-w_t^{pr}\|^2\\
  \overset{(c)}{\leq}& (1-\alpha_t)H(w_{t-1}^{ag})+\alpha_tH(w)-H(w_t^{ag})+\alpha_t\left(\frac{1}{2\eta_t}-\frac{1}{2}\right)[\|w_{t-1}-w\|^2-\|w_t-w\|^2]\\
  &+\frac{1+\alpha_t}{2}\|\delta_t\|^2-\frac{1}{2\beta_t}\|w_t^{ag}-w_t^{pr}\|^2
\end{align*}
where (a) holds because of the convexity of $H$, (b) is true because $w_t^{pr}=(1-\alpha_t)w_{t-1}^{ag}+\alpha_t w_{t-1}$, and (c) is true because $\alpha_t\eta_t\leq\beta_t$.

Continuing with \eqref{firsttwo}, we can have that
\begin{align*}
    G(w_t^{ag})\leq& (1-\alpha_t)G(w_{t-1}^{ag})+\alpha_t G(w)+(1-\alpha_t)H(w_{t-1}^{ag})+\alpha_t H(w)-H(w_t^{ag})\\
    &+\frac{\lambda \rho}{2(\lambda+\rho)}\|w_t^{ag}-w_t^{pr}\|^2+\frac{\lambda \rho(1-\alpha_t)\alpha_t^2}{2(\lambda+\rho)}\|w_{t-1}^{ag}-w_{t-1}\|^2+\frac{\lambda \rho\alpha_t}{2(\lambda+\rho)}\|w-w_t^{pr}\|^2\\
    &+\alpha_t\left(\frac{1}{2\eta_t}-\frac{1}{2}\right)[\|w_{t-1}-w\|^2-\|w_t-w\|^2]+\frac{1+\alpha_t}{2}\|\delta_t\|^2-\frac{1}{2\beta_t}\|w_t^{ag}-w_t^{pr}\|^2\\
    =&(1-\alpha_t)G(w_{t-1}^{ag})+\alpha_t G(w)+(1-\alpha_t)H(w_{t-1}^{ag})+\alpha_t H(w)-H(w_t^{ag})\\
    &+\left(\frac{\lambda \rho}{2(\lambda+\rho)}-\frac{1}{2\beta_t}\right)\|w_t^{ag}-w_t^{pr}\|^2+\frac{\lambda \rho(1-\alpha_t)\alpha_t^2}{2(\lambda+\rho)}\|w_{t-1}^{ag}-w_{t-1}\|^2+\frac{\lambda \rho\alpha_t}{2(\lambda+\rho)}\|w-w_t^{pr}\|^2\\
    &+\alpha_t\left(\frac{1}{2\eta_t}-\frac{1}{2}\right)[\|w_{t-1}-w\|^2-\|w_t-w\|^2]+\frac{1+\alpha_t}{2}\|\delta_t\|^2
\end{align*}
which indicates that
\begin{align*}
    F(w_t^{ag})\leq& (1-\alpha_t)F(w_{t-1}^{ag})+\alpha_t F(w)+\left(\frac{\lambda \rho}{2(\lambda+\rho)}-\frac{1}{2\beta_t}\right)\|w_t^{ag}-w_t^{pr}\|^2+\frac{\lambda \rho(1-\alpha_t)\alpha_t^2}{2(\lambda+\rho)}\|w_{t-1}^{ag}-w_{t-1}\|^2\\
    &+\frac{\lambda \rho\alpha_t}{2(\lambda+\rho)}\|w-w_t^{pr}\|^2+\alpha_t\left(\frac{1}{2\eta_t}-\frac{1}{2}\right)[\|w_{t-1}-w\|^2-\|w_t-w\|^2]+\frac{1+\alpha_t}{2}\|\delta_t\|^2.
\end{align*}

Based on Lemma \ref{lemmaaux1}, it follows that
\begin{align}\label{sum}
    &\frac{F(w_T^{ag})-F(w)}{\Gamma_T}\nonumber\\
    \leq& \sum_{t=1}^T\left\{\frac{1}{\Gamma_t}\left(\frac{\lambda \rho}{2(\lambda+\rho)}-\frac{1}{2\beta_t}\right)\|w_t^{ag}-w_t^{pr}\|^2+\frac{\lambda \rho(1-\alpha_t)\alpha_t^2}{2\Gamma_t(\lambda+\rho)}\|w_{t-1}^{ag}-w_{t-1}\|^2\right\}\nonumber\\
    &+\sum_{t=1}^T\left\{\frac{\lambda \rho\alpha_t}{2\Gamma_t(\lambda+\rho)}\|w-w_t^{pr}\|^2+\frac{\alpha_t}{\Gamma_t}\left(\frac{1}{2\eta_t}-\frac{1}{2}\right)[\|w_{t-1}-w\|^2-\|w_t-w\|^2]+\frac{1+\alpha_t}{2\Gamma_t}\|\delta_t\|^2\right\}\nonumber\\
    \leq& \sum_{t=1}^T\left\{\frac{1}{\Gamma_t}\left(\frac{\lambda \rho}{2(\lambda+\rho)}-\frac{1}{2\beta_t}\right)\|w_t^{ag}-w_t^{pr}\|^2+\frac{\lambda \rho(1-\alpha_t)\alpha_t^2}{2\Gamma_t(\lambda+\rho)}\|w_{t-1}^{ag}-w_{t-1}\|^2\right\}\nonumber\\
    &+\sum_{t=1}^T\left\{\frac{\lambda \rho\alpha_t}{2\Gamma_t(\lambda+\rho)}\|w-w_t^{pr}\|^2+\frac{1+\alpha_t}{2\Gamma_t}\|\delta_t\|^2\right\}+\left(\frac{1}{2\eta_1}-\frac{1}{2}\right)\|w_0-w\|^2
\end{align}
where the last inequality is true because
\begin{align*}
    &\sum_{t=1}^T\frac{\alpha_t}{\Gamma_t}\left(\frac{1}{2\eta_t}-\frac{1}{2}\right)[\|w_{t-1}-w\|^2-\|w_t-w\|^2]\\
    \leq& \frac{\alpha_1}{\Gamma_1}\left(\frac{1}{2\eta_1}-\frac{1}{2}\right)\sum_{t=1}^T [\|w_{t-1}-w\|^2-\|w_t-w\|^2]\\
    =&\left(\frac{1}{2\eta_1}-\frac{1}{2}\right)[\|w_0-w\|^2-\|w_T-w\|^2]\\
    \leq& \left(\frac{1}{2\eta_1}-\frac{1}{2}\right)\|w_0-w\|^2
\end{align*}
considering that $\frac{\alpha_t}{\Gamma_t}\left(\frac{1}{2\eta_t}-\frac{1}{2}\right)\geq\frac{\alpha_{t+1}}{\Gamma_{t+1}}\left(\frac{1}{2\eta_{t+1}}-\frac{1}{2}\right)$ for any $t\geq 1$.

Therefore, by rearranging \eqref{sum}, we can have that for any $w$ the following holds
\begin{align*}
    &\sum_{t=1}^T\frac{1}{\Gamma_t}\left(\frac{1}{2\beta_t}-\frac{\lambda \rho}{2(\lambda+\rho)}\right)\|w_t^{ag}-w_t^{pr}\|^2\\
    \leq& \frac{F(w)-F(w_T^{ag})}{\Gamma_T}+\sum_{t=1}^T\frac{\lambda \rho\alpha_t}{2\Gamma_t(\lambda+\rho)}[(1-\alpha_t)\alpha_t\|w_{t-1}^{ag}-w_{t-1}\|^2+\|w-w_t^{pr}\|^2]\\
    &+\left(\frac{1}{2\eta_1}-\frac{1}{2}\right)\|w_0-w\|^2+\sum_{t=1}^T \frac{1+\alpha_t}{2\Gamma_t}\|\delta_t\|^2\\
    \leq&\frac{F(w)-F(w_T^{ag})}{\Gamma_T}+\sum_{t=1}^T\frac{\lambda \rho\alpha_t}{\Gamma_t(\lambda+\rho)}[\alpha_t(1-\alpha_t)(\|w_{t-1}^{ag}\|^2+\|w_{t-1}\|^2)+\|w\|^2+\|w_t^{pr}\|^2]\\
    &+\left(\frac{1}{2\eta_1}-\frac{1}{2}\right)\|w_0-w\|^2+\sum_{t=1}^T \frac{1+\alpha_t}{2\Gamma_t}\|\delta_t\|^2\\
    \overset{(a)}{\leq}&\frac{F(w)-F(w_T^{ag})}{\Gamma_T}+\sum_{t=1}^T\frac{\lambda \rho\alpha_t}{\Gamma_t(\lambda+\rho)}[\alpha_t(1-\alpha_t)(\|w_{t-1}^{ag}\|^2+\|w_{t-1}\|^2)+\|w\|^2+(1-\alpha_t)\|w_{t-1}^{ag}\|^2+\alpha_t\|w_{t-1}\|^2]\\
    &+\left(\frac{1}{2\eta_1}-\frac{1}{2}\right)\|w_0-w\|^2+\sum_{t=1}^T \frac{1+\alpha_t}{2\Gamma_t}\|\delta_t\|^2\\
    \leq& \frac{F(w)-F(w_T^{ag})}{\Gamma_T}+\sum_{t=1}^T\frac{\lambda \rho\alpha_t}{\Gamma_t(\lambda+\rho)}[\|w\|^2+\|w_{t-1}^{ag}\|^2+\|w_{t-1}\|^2]+\left(\frac{1}{2\eta_1}-\frac{1}{2}\right)\|w_0-w\|^2+\sum_{t=1}^T \frac{1+\alpha_t}{2\Gamma_t}\|\delta_t\|^2\\
    \overset{(b)}{\leq}&\frac{F(w)-F(w_T^{ag})}{\Gamma_T}+\sum_{t=1}^T\frac{\lambda \rho\alpha_t}{\Gamma_t(\lambda+\rho)}[\|w\|^2+2M]+\left(\frac{1}{2\eta_1}-\frac{1}{2}\right)\|w_0-w\|^2+\sum_{t=1}^T \frac{1+\alpha_t}{2\Gamma_t}\|\delta_t\|^2\\
    \overset{(c)}{=}&\frac{F(w)-F(w_T^{ag})}{\Gamma_T}+\frac{\lambda \rho}{\Gamma_T(\lambda+\rho)}(\|w\|^2+2M)+\left(\frac{1}{2\eta_1}-\frac{1}{2}\right)\|w_0-w\|^2+\sum_{t=1}^T \frac{1+\alpha_t}{2\Gamma_t}\|\delta_t\|^2
\end{align*}
where (a) is because of the convexity of $\|\cdot\|^2$ and $w_t^{pr}=(1-\alpha_t)w_{t-1}^{ag}+\alpha_t w_{t-1}$, (b) is because the solution to the proximal optimization problem is bound from above, and (c) holds because
\begin{align*}
    \sum_{t=1}^T\frac{\alpha_t}{\Gamma_t}=\frac{1}{\Gamma_1}+\sum_{t=2}^T \left(\frac{1}{\Gamma_t}-\frac{1}{\Gamma_{t-1}}\right)=\frac{1}{\Gamma_T}.
\end{align*}

Let $w^*$ be the optimal solution of $F(w)$, i.e., $F(w^*)=\min_w F(w)$. Since $\mathcal{Q}(w_t^{pr},\nabla_w G(w_t^{pr}),\beta_t)=\frac{1}{\beta_t}(w_t^{pr}-w_t^{ag})$, it is clear that
\begin{align*}
    &\sum_{t=1}^T\frac{\beta_t^2}{\Gamma_t}\left(\frac{1}{2\beta_t}-\frac{\lambda \rho}{2(\lambda+\rho)}\right)\|\mathcal{Q}(w_t^{pr},\nabla_w G(w_t^{pr}),\beta_t)\|^2\\
    \leq& \frac{F(w^*)-F(w_T^{ag})}{\Gamma_T}+\frac{\lambda \rho}{\Gamma_T(\lambda+\rho)}(\|w^*\|^2+2M)+\left(\frac{1}{2\eta_1}-\frac{1}{2}\right)\|w_0-w^*\|^2+\sum_{t=1}^T \frac{1+\alpha_t}{2\Gamma_t}\|\delta_t\|^2\\
    \leq& \frac{\lambda \rho}{\Gamma_T(\lambda+\rho)}(\|w^*\|^2+2M)+\left(\frac{1}{2\eta_1}-\frac{1}{2}\right)\|w_0-w^*\|^2+\sum_{t=1}^T \frac{1+\alpha_t}{2\Gamma_t}\|\delta_t\|^2
\end{align*}
which implies that
\begin{align*}
    &\min_{t\in[1,T]} \|\mathcal{Q}(w_t^{pr},\nabla_w G(w_t^{pr}),\beta_t)\|^2\\
    \leq&\left[\sum_{t=1}^T\frac{\beta_t^2}{\Gamma_t}\left(\frac{1}{2\beta_t}-\frac{\lambda \rho}{2(\lambda+\rho)}\right)\right]^{-1}\left[\frac{\lambda \rho}{\Gamma_T(\lambda+\rho)}(\|w^*\|^2+2M)+\left(\frac{1}{2\eta_1}-\frac{1}{2}\right)\|w_0-w^*\|^2+\sum_{t=1}^T \frac{1+\alpha_t}{2\Gamma_t}\|\delta_t\|^2\right].
\end{align*}

For convenience, choose $\alpha_t=\frac{2}{t+1}$, $\beta_t<\frac{\lambda+\rho}{\lambda \rho}$, and $\eta_t$ to satisfy that
\begin{align*}
    \alpha_t\eta_t\leq \beta_t, ~ \frac{\alpha_t}{\Gamma_t}\left(\frac{1}{\eta_t}-1\right)\geq \frac{\alpha_{t+1}}{\Gamma_{t+1}}\left(\frac{1}{\eta_{t+1}}-1\right),
\end{align*}
for any $t\in[1,T]$. Then we can have
\begin{align*}
  & \min_{t\in[1,T]} \|\mathcal{Q}(w_t^{pr},\nabla_w G(w_t^{pr}),\beta_t)\|^2\\
  \leq& \frac{48\lambda \rho}{T^2(T+1)(\lambda+\rho)}\left[\frac{T(T+1)\lambda \rho}{2(\lambda+\rho)}(\|w^*\|^2+2M)+\left(\frac{1}{2\eta_1}-\frac{1}{2}\right)\|w_0-w^*\|^2+\frac{T(T+1)}{2}\sum_{t=1}^T\|\delta_t\|^2\right]\\
  =&\frac{24(\lambda \rho)^2}{T(\lambda+\rho)^2}(\|w^*\|^2+2M)+\frac{C}{T^2(T+1)}\|w_0-w^*\|^2+\frac{24\lambda \rho}{T(\lambda+\rho)}\sum_{t=1}^T\|\delta_t\|^2\\
  \leq&\frac{24(\lambda \rho)^2}{T(\lambda+\rho)^2}(\|w^*\|^2+2M)+\frac{C}{T^2(T+1)}\|w_0-w^*\|^2+\frac{24\lambda^3 \rho}{T(\lambda+\rho)(\lambda-\rho)^2}\sum_{t=1}^T
  \xi_t.
\end{align*}
\end{proof}

\section{More Experimental Results}

\subsection{Experimental setup}

To obtain a comprehensive understanding of the performance of MetaGater, we conduct more experiments of image classification on CIFAR-10 and CIFAR-100 \cite{krizhevsky2009learning} by using a DNN with more complex structure. Specifically, as shown in Fig. \ref{Fig:structure}, we consider a seven-layer convolutional neural network with six convolutional layers, followed by a fully connected layer. Each convolution layer is  followed by a ReLU activation layer. Besides, we also adopt two max pooling layers and one average pooling layer to shrink the feature map dimension.
And we integrate a channel gating module with the fifth and sixth convolutional layers, respectively.

\begin{figure}
\centering
\includegraphics[scale=0.18]{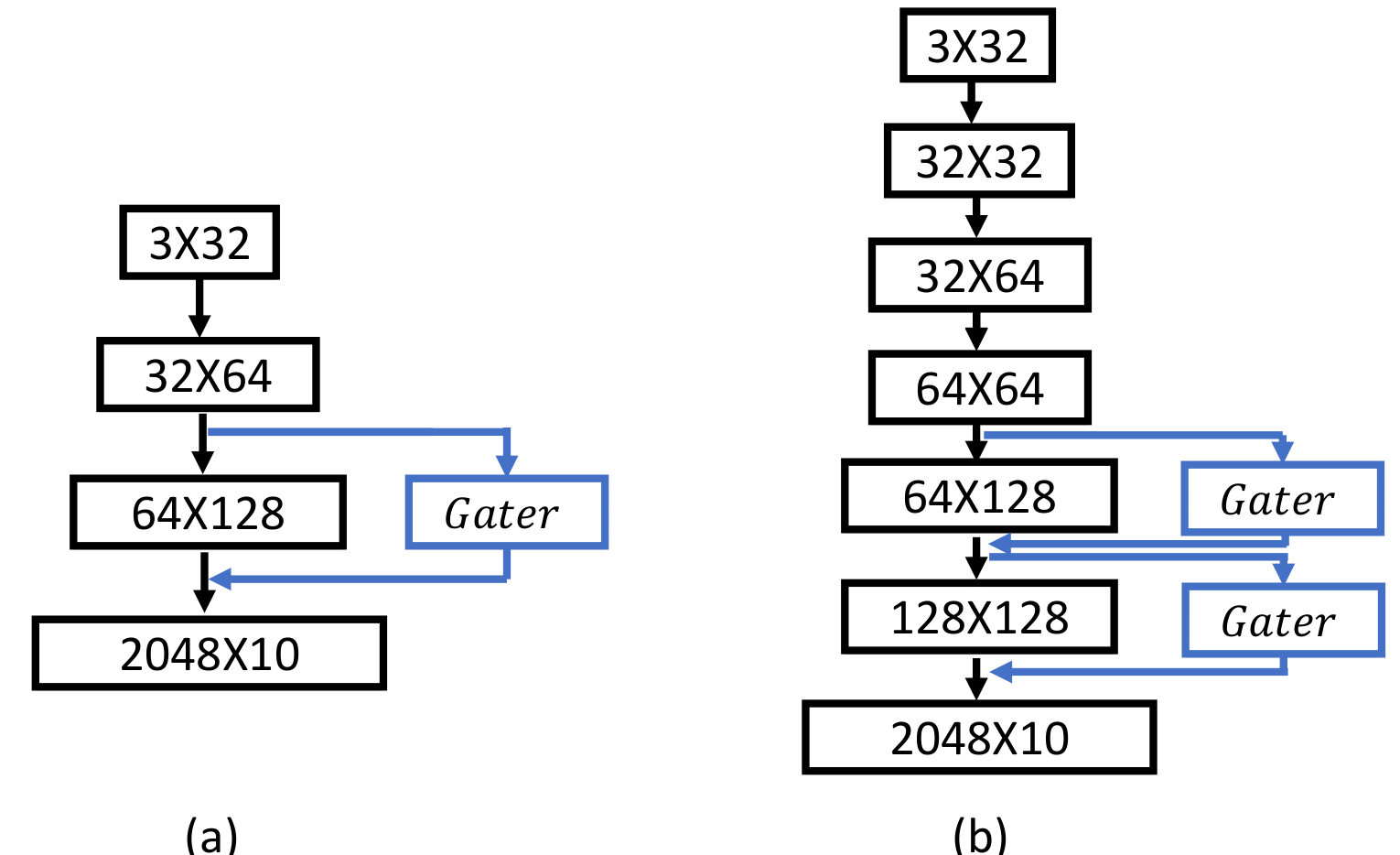}
\caption{The DNN architectures used in the experiments. We use (a) for experiments in Table \ref{tab:fedmeta}-\ref{tab:gating}, and (b) for experiments in Table \ref{tab:fedmeta-large}-\ref{tab:adapt-small}.}
\label{Fig:structure}
\end{figure}

It is worth noting that we integrate the channel gating module only at the layers near the output layer, aiming to 1) minimize the introduced model overhead, and 2) preserve the more important features in the layers near the input layer. In fact, we have also tried to integrate the channel gating module only at the second convolutional layer of the model (a) shown in Fig. \ref{Fig:structure}, and the testing accuracy degrades compared with the the marginal accuracy drop in Table \ref{tab:gating}. Such a performance degradation makes sense as the low-level features captured by the first few layers have more critical impact on the overall performance, in contrast to the high-level features in the last few layers.
Besides, we study the cross-entropy loss and use Group Lasso to prompt the sparsity of the channel gating module. 

Since a resource-limited node often has only a small local dataset for learning, 
we further investigate the performance of MetaGater under different sizes of local datasets. In particular, we consider two different regimes of the local dataset size:
\begin{itemize}
    \item \emph{Moderate local dataset size.} For CIFAR-10, we distribute the dataset among $N=20$ nodes, where each node only has the data samples from two classes and
    the number of samples per node is in the range of $[221, 2792]$. We randomly select $10$ nodes for each round in federated meta-learning, and $10$ target nodes for fast adaptation. 
    \item \emph{Small local dataset size.} For both CIFAR-10 and CIFAR-100, we distribute the dataset among $N=20$ nodes, where
    the number of samples per node is in the range of $[30, 100]$.
    We randomly select $10$ nodes for each round in federated meta-learning, and $10$ target nodes for fast adaptation. 
    The difference is that each node has the data samples from two classes in CIFAR-10, but five classes in CIFAR-100.
\end{itemize}

We evaluate the performance under a fixed number of communication rounds, i.e., $T=400$.  For $t$-th round,
the learning rate $\alpha_t=\frac{2}{t+1}$, and we choose $\beta_t=\alpha_t\eta_t=1$. Besides, $\lambda=0.2$. During federated meta-learning, we run gradient descent for one local update to solve the local minimization problem \eqref{localmin} for each training task.  For fast adaptation at target nodes, we only run one-step gradient descent to fine-tune both meta-backbone network and meta-gating module. We 
evaluate the testing accuracy at the target nodes,
and repeat all the experiments for 5 times to obtain the average performance.

\subsection{Performance of MetaGater with moderate local datasets}

We evaluate the performance of MetaGater with moderate local datasets on CIFAR-10. As shown in Table \ref{tab:fedmeta-large}, MetaGater still achieves the best accuracy performance. It is also worth to mention that for MetaGater with gating,
even though  extra training time is needed to jointly train the meta-backbone network and the meta-gating module compared to training MetaGater without gating, it still takes the similar training time but achieves much higher accuracy, in contrast to Per-FedAvg. Moreover, it can be seen from Table \ref{tab:adapt-large} that with channel gating module, the target node is able to quickly obtain a more compact model for efficient inference with almost same accuracy performance, compared with MetaGater without gating module. And  MetaGater clearly outperforms MetaSNIP in terms of the testing accuracy and demonstrates a larger diversity of the achieved model sparsity, in a similar speed with MetaSNIP.

\begin{table}
    \begin{center}
    \begin{tabular}{c|cccc}
         \toprule
         Dataset &
          Method &
               Local 
               update 
          & Accuracy(\%) & 
               Training  
               time(s)\\
          \hline
          \multirow{4}{*}{CIFAR-10}&
         FedAvg
            & 1 & $ 74.2\pm1.2 $ & 1677\\
         \cline{2-5}
         &
         Per-FedAvg
            & 1 & $ 76.8\pm1.8 $ & 3098\\
        \cline{2-5}
        &
               MetaGater
               w/o gating
            & 1 &  $\pmb{85.3\pm2.1}$  & 1939\\
            \cline{2-5}
         &
               MetaGater 
               w/ gating
             & 1 & $ \pmb{85.1\pm3.6} $ & 2558\\
            \hline
    \end{tabular}
    \end{center}
    \caption{Accuracy comparison for MetaGater, Per-FedAvg, FedAvg on CIFAR-10 with moderate local datasets. Clearly, MetaGater achieves the best accuracy among all methods. }
    \label{tab:fedmeta-large}
\end{table}

\begin{table}
    \begin{center}
    \begin{tabular}{c|cccc}
         \toprule
         Dataset &
          Method &
           Accuracy(\%) &
          Sparsity(\%) & 
               Learning  
               time(s)\\
          \hline
          \multirow{3}{*}{CIFAR-10}&
         MetaGater w/o gating
            & $ \pmb{85.3\pm2.1} $ & $ 0 $ & 1.6\\
         \cline{2-5}
         &
         MetaSNIP
            & $ 83.4\pm3.4 $ & $ 23\pm2.4 $ & 1.5\\
        \cline{2-5}
         &
               MetaGater 
               w/ gating
             & $ \pmb{85.1\pm3.6} $ & $ 23\pm3.7 $ & 1.4\\
            \hline
    \end{tabular}
    \end{center}
    \caption{Fast adaptation performance comparison on CIFAR-10 with moderate local datasets. Compared with MetaSNIP, MetaGater has a better accuracy and a larger sparsity range.}
    \label{tab:adapt-large}
\end{table}

\subsection{Performance of MetaGater with small local datasets}

We evaluate the performance of MetaGater with small local datasets on both CIFAR-10 and CIFAR-100. As illustrated in Table \ref{tab:fedmeta-small}, MetaGater significantly outperforms other federated meta-learning methods, and the performance gap is even larger compared with the case under moderate local datasets. And Table \ref{tab:adapt-small} shows that MetaGater still performs better than MetaSNIP. More importantly, as the model becomes deeper and the local datasets becomes smaller, the accuracy gap between MetaGater w/ gating and MetaGater w/o gating decreases, and MetaGater w/ gating can even achieve a better accuracy compared with MetaGater w/o gating. This is because the backbone network will become relatively overparameterized, and the channel gating module can accurately select the task-specific subnet with the data-dependent important filters that lead to a better accuracy performance. 

In a nutshell, under different model sizes and various local dataset sizes, MetaGater clearly achieves better accuracy compared with other federated meta-learning methods, and outperforms MetaSNIP, in the sense that a better accuracy and a larger sparsity diversity can be attained by MetaGater in a similar speed with MetaSNIP. Moreover, since the meta-gating module is able to effectively capture the important filters of a good meta-backbone network and hence sparsity structure across tasks, a task-specific conditional channel gated network for a new task can be quickly adapted from the meta-initialization with only a small local dataset. In this way, the obtained channel gated network with structural sparsity patterns significantly reduces the network size for computing while still guaranteeing the testing accuracy, by quickly selecting the important filters for the new task.

\begin{table}
    \begin{center}
    \begin{tabular}{c|cccc}
         \toprule
         Dataset &
          Method &
               Local  
               update 
          & Accuracy(\%) & 
               Training  
               time(s)
          \\
          \hline
          \multirow{4}{*}{CIFAR-10}&
         FedAvg
            & 1 & $ 58.7\pm1.2 $ & 187\\
         \cline{2-5}
         &
         Per-FedAvg
            & 1 & $ 60.5\pm1.9 $ & 339\\
        \cline{2-5}
        &
              MetaGater
              w/o gating
            & 1 & $ \pmb{78.6\pm2.3} $ & 189\\
            \cline{2-5}
         &
              MetaGater 
              w/ gating
             & 1 & $ \pmb{78.4\pm3.5} $ & 356\\
            \hline
            \multirow{4}{*}{CIFAR-100}&
        FedAvg
            & 1 & $ 42.8\pm1.1 $ & 192\\
         \cline{2-5}
         &
         Per-FedAvg
            & 1 & $ 56.6\pm1.6 $ & 367\\
        \cline{2-5}
        &
               MetaGater  
               w/o gating
            & 1 & $ \pmb{68.4\pm2.1} $ & 208\\
                    \cline{2-5}
         &
               MetaGater 
               w/ gating
             & 1 & $ \pmb{68.6\pm4.5} $ & 399\\
            \hline
    \end{tabular}
    \end{center}
    \caption{Accuracy comparison for MetaGater, Per-FedAvg, FedAvg on CIFAR-10 and CIFAR-100 with small local datasets. Clearly, MetaGater achieves the best accuracy among all methods. }
    \label{tab:fedmeta-small}
\end{table}

\begin{table}
    \begin{center}
    \begin{tabular}{c|cccc}
         \toprule
         Dataset &
          Method &
           Accuracy(\%) &
          Sparsity(\%) & 
               Learning  
               time(s)\\
          \hline
          \multirow{3}{*}{CIFAR-10}&
         MetaGater w/o gating
            & $ \pmb{78.6\pm2.3} $ & $ 0 $ & 0.56 \\
         \cline{2-5}
         &
         MetaSNIP
            & $ 76.7\pm2.4 $ & $ 17\pm1.3 $ & 0.55\\
        \cline{2-5}
         &
              MetaGater 
              w/ gating
             & $ \pmb{78.4\pm3.5} $ & $ 17\pm2.1 $ & 0.51\\
            \hline
            \multirow{3}{*}{CIFAR-100}&
        MetaGater w/o gating
            & $ \pmb{68.4\pm2.1} $ & $ 0 $ & 0.61 \\
         \cline{2-5}
         &
         MetaSNIP
            & $ 66.9\pm3.7 $ & $ 21\pm2.3 $ & 0.57\\
        \cline{2-5}
         &
               MetaGater 
               w/ gating
             & $ \pmb{68.6\pm4.5} $ & $ 21\pm3.6 $ & 0.52 \\
            \hline
    \end{tabular}
    \end{center}
    \caption{Fast adaptation performance comparison on CIFAR-10 and CIFAR-100 with small local datasets. Compared with MetaSNIP, MetaGater has a better accuracy and a larger sparsity range. }
    \label{tab:adapt-small}
\end{table}

\section{Channel Gating Module}
Fig. \ref{Fig:fml_att_2} shows the channel gating module in detail. To generate a binary mask, a straightforward way is to use a binarization function, which utilizes a hard threshold to take binary on/off decision. However, such  a discrete function is non-differential during back-propagation. A widely used solution is straight-through estimator \cite{bengio2013estimating} where the incoming gradient is equal to the outgoing gradient. To better estimate the gradient, we use Gumbel Softmax trick \cite{jang2016categorical}. Specifically, we utilize the hard threshold during forward pass to generate binary masks and the differential softmax function during back-propagation. Note that, the temperature of the Gumbel Softmax is set to be 1 for all the experiments.

 \begin{figure}
\centering
\includegraphics[scale=0.18]{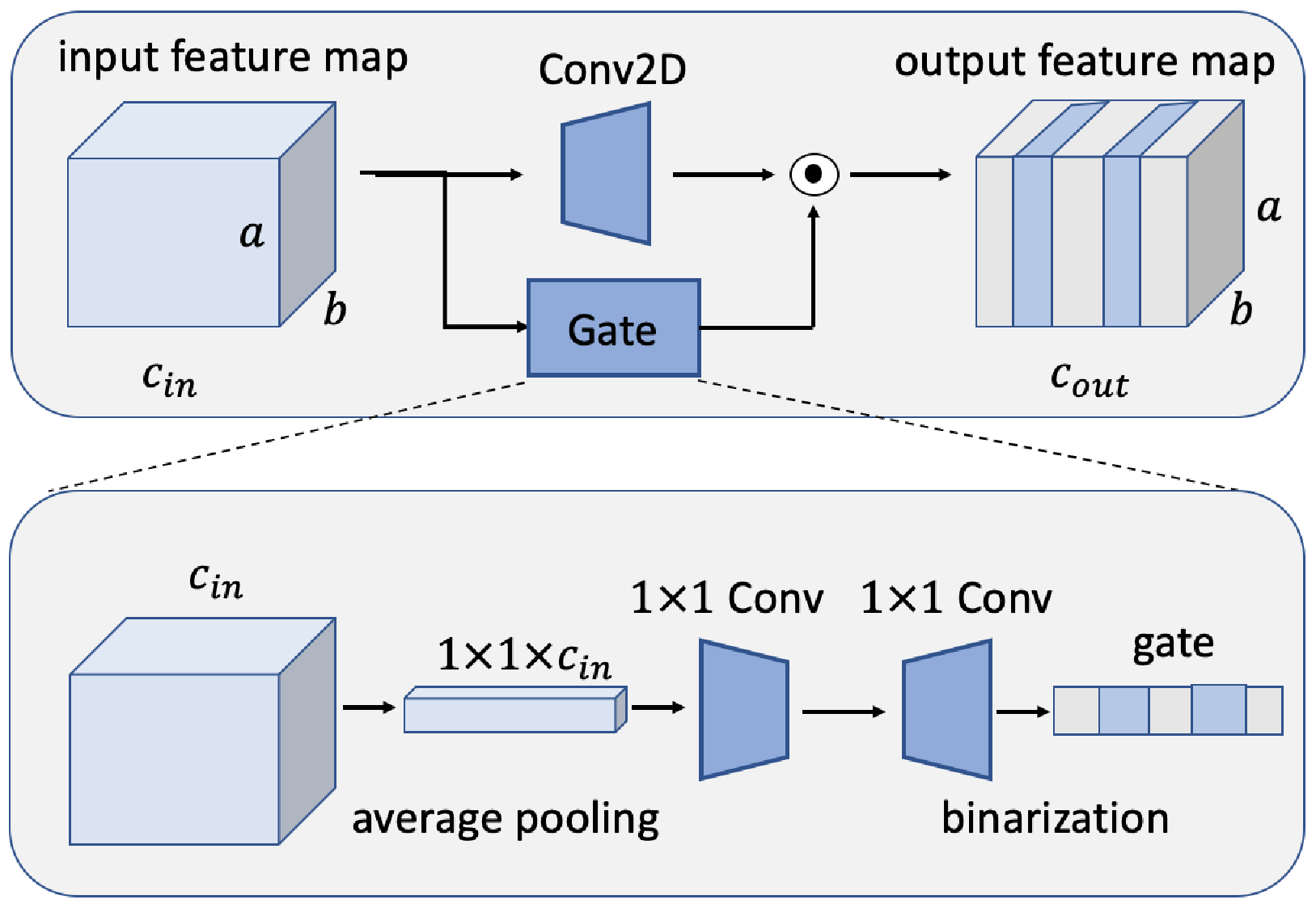}
\caption{The channel gating module for a convolution layer.}
\label{Fig:fml_att_2}
\end{figure}

\end{document}